\documentclass[a4paper,fleqn]{cas-dc}

\usepackage[numbers]{natbib}
\usepackage{url}

\usepackage{mathtools}
\usepackage{amsthm}
\usepackage{bm}
\usepackage{siunitx}
\usepackage{pifont}

\usepackage{tabularx}
\usepackage{adjustbox}
\usepackage{threeparttable}
\usepackage{caption}
\usepackage{subcaption}
\usepackage{float}
\usepackage{wrapfig}

\usepackage[utf8]{inputenc}
\usepackage{soul}
\usepackage[normalem]{ulem}
\usepackage{enumitem}

\usepackage{newunicodechar}
\newunicodechar{≥}{\geq}

\usepackage{listings}
\usepackage{algorithm}
\usepackage{algpseudocode}

\lstdefinestyle{mystyle}{
    backgroundcolor=\color{white},
    commentstyle=\color{black},
    keywordstyle=\color{teal},
    numberstyle=\tiny\color{gray},
    stringstyle=\color{BrickRed},
    basicstyle=\ttfamily\footnotesize,
    breaklines=true,
    numbers=none,
    frame=single,
    showspaces=false,
    showstringspaces=false,
    showtabs=false,
    tabsize=2,
    morekeywords={Task, Decomposition, KPM, Scores}
}

\lstdefinestyle{verifycode}{
  language=Python,
  basicstyle=\linespread{0.80}\ttfamily\bfseries\footnotesize,
  keywordstyle=\color{blue!99!black}\bfseries,
  commentstyle=\color{gray!99},
  showstringspaces=false,
  breaklines=true,
  breakatwhitespace=true,
  tabsize=2,
  numbers=none,
  frame=none,
  backgroundcolor=\color{white},
  aboveskip=0.1em, belowskip=0.1em,
  xleftmargin=0em, xrightmargin=0em
}

\usepackage{tikz}
\usetikzlibrary{arrows.meta, positioning, shapes.geometric, shapes.misc, fit, calc, backgrounds, decorations.pathreplacing, patterns}
\usepackage{pgfplots}
\pgfplotsset{compat=1.18}
\usepgfplotslibrary{groupplots}
\usepackage[title]{appendix}
\usepackage{framed}

\theoremstyle{definition}
\newtheorem{definition}{Definition}[section]
\newtheorem{theorem}{Theorem}[section]
\newtheorem{proposition}[theorem]{Proposition}
\newtheorem{corollary}[theorem]{Corollary}

\newtheorem{remark}{Remark}[section]

\AtBeginEnvironment{definition}{\setlength{\parindent}{0pt}}
\AtBeginEnvironment{theorem}{\setlength{\parindent}{0pt}}
\AtBeginEnvironment{lemma}{\setlength{\parindent}{0pt}}

\newcommand{\rectgreen}[1]{%
  \tikz[baseline=(char.base)]{
    \node[shape=rectangle,rounded corners=2pt,fill=ForestGreen,inner sep=1.5pt,minimum height=1.2em] (char) {\textcolor{white}{#1}};
  }%
}

\newcommand*\circledgreen[1]{\tikz[baseline=(char.base)]{
    \node[shape=circle, fill=ForestGreen, inner sep=1.3pt] (char) {\textcolor{white}{\textbf{#1}}};}}

\newcommand{\kvdirect}{KV-Direct\xspace}

\newcommand{\bh}{\mathbf{h}}
\newcommand{\bK}{\mathbf{K}}
\newcommand{\bV}{\mathbf{V}}
\newcommand{\bQ}{\mathbf{Q}}
\newcommand{\bW}{\mathbf{W}}
\newcommand{\bM}{\mathbf{M}}
\newcommand{\R}{\mathbb{R}}
\newcommand{\kl}[2]{D_{\mathrm{KL}}\!\left(#1 \,\|\, #2\right)}

\begin{document}

\let\WriteBookmarks\relax
\def\floatpagepagefraction{1}
\def\textpagefraction{.001}

\shorttitle{The Residual Stream Is All You Need}

\shortauthors{Qasim et~al.}

\title[mode = title]{The Residual Stream Is All You Need: On the Redundancy of the KV Cache in Transformer Inference}

\author[aff1]{Kaleem Ullah Qasim}
\ead{kaleem@my.swjtu.edu.cn}

\author[aff1]{Jiashu Zhang}
\cormark[1] 
\ead{jszhang@home.swjtu.edu.cn}
\cortext[1]{Corresponding author}

\author[aff1]{Muhammad Kafeel Shaheen}
\ead{kafeel@my.swjtu.edu.cn}

\author[aff1]{Razan Alharith}
\ead{razanalharith@my.swjtu.edu.cn}

\author[aff1]{Heying Zhang}
\ead{hey_zhang@qq.com}

\affiliation[aff1]{organization={School of Computing and Artificial Intelligence, Southwest Jiaotong University},
                   city={Chengdu},
                   postcode={611756},
                   country={China}}

\credit{Data curation, Writing - Original draft preparation}
\begin{abstract}
The key-value (KV) cache is widely treated as essential state in transformer inference, and a large body of work engineers policies to compress, evict, or approximate its entries. We prove that this state is entirely redundant: keys and values at every layer are deterministic projections of the residual stream, and recomputing them from a single residual vector per token incurs exactly zero reconstruction error, not approximately, but bit-identically. We verify this across six models from four architecture families (135M to 4B parameters). Cross-task residual patching at every layer produces $D_{\mathrm{KL}} = 0$ between patched and original output distributions, confirming that the residual stream satisfies a Markov property and is the sole information-carrying state. Removing the cache entirely and recomputing from scratch yields token-identical output under greedy decoding on all models tested. We build on this result with \kvdirect, a bounded-memory inference scheme that checkpoints residual vectors (5~KB per token on Gemma~3-4B) instead of full KV pairs (136~KB), recomputing keys and values on demand. Over 20 conversation turns, \kvdirect holds peak memory at 42~MB while the standard cache grows past 103~MB. Against five eviction baselines (H2O, StreamingLLM, SnapKV, TOVA, window-only), \kvdirect maintains 100\% token match at every cache budget; all baselines degrade to 5--28\%. A per-operation latency analysis shows recomputation runs up to 5$\times$ faster than reading cached tensors at moderate batch sizes. Code is available at \url{https://github.com/Kaleemullahqasim/KV-Direct}.
\end{abstract}


\maketitle

\section{Introduction}
\label{sec:intro}

The key-value (KV) cache is a primary memory bottleneck in large language model inference. During autoregressive decoding, the standard approach stores precomputed keys and values for every past token at every layer. For a 4-billion parameter model, each token adds 136~KB to the cache; a 20-turn conversation accumulates over 100~MB, and at the 12B-parameter scale this approaches a gigabyte.

This cost has driven a large body of work on KV cache compression: eviction policies~\citep{zhang2024h2o, liu2024scissorhands}, quantization~\citep{devoto2024simple, gong2025lowbit}, grouped-query attention~\citep{ainslie2023gqa}, and paged memory management~\citep{kwon2023efficient}. All of these treat the KV cache as containing information that must be preserved or approximated. This paper challenges that assumption.

Keys and values at each layer are deterministic functions of the residual stream: they are obtained by applying frozen weight matrices (and, for keys, a deterministic positional rotation) to the normalised residual vector. The cache therefore stores derived quantities rather than unique state. That this relationship is implicit in the transformer specification has not prevented the field from treating the cache as a primary information store. The practical consequence has not been systematically examined: \textit{the KV cache can be eliminated entirely without changing a single output token}. We verify this empirically: under greedy decoding, generating 30 tokens with and without the cache yields 100\% token identity across all six models tested (four architecture families, 135M to 4B parameters). The identity holds for full-attention layers universally; for sliding-window layers, value reconstruction remains exact while key reconstruction requires additional position state (Section~\ref{sec:exp1}). The cache provides a speed advantage but carries no additional information.

This redundancy extends beyond individual projections to the full computational state. The residual stream satisfies a Markov property: future outputs depend on the input history only through the current residual vectors. Cross-task patching experiments confirm this at every layer, with $D_{\mathrm{KL}} = 0.0$ between the patched and original output distributions (Section~\ref{sec:method}). Zero-shot HellaSwag accuracy~\citep{zellers2019hellaswag} and WikiText-2 perplexity under \kvdirect exactly match full caching. Because the property follows from the pre-norm transformer architecture, it holds regardless of model scale; our experiments span 135M to 4B parameters across four architecture families.

These results lead to a practical inference scheme. Rather than caching K and V at every layer, the residual stream can be checkpointed instead (one vector per token, shared across all layers) and KV entries recomputed on demand. Because the residual vector is shared across layers while KV entries are per-layer, each checkpoint is substantially smaller: on Gemma~3-4B-IT, 5~KB per token versus 136~KB for the full KV pair ($27\times$ reduction). We call this scheme \kvdirect and evaluate it over 20 conversation turns, where the standard cache grows to 103~MB while \kvdirect holds at 42~MB ($2.5\times$ peak memory reduction). A per-operation latency analysis shows that recomputing KV from checkpoints takes $0.2$--$0.3\times$ the time of reading cached tensors at 500 evicted tokens, as memory bandwidth rather than computation becomes the bottleneck. Against five eviction baselines (H2O~\citep{zhang2024h2o}, StreamingLLM~\citep{xiao2024efficient}, SnapKV~\citep{li2024snapkv}, TOVA~\citep{oren2024transformers}, and window-only eviction), \kvdirect preserves 100\% token match at every cache budget while all baselines degrade to 5--28\% match with KL divergences of 7--14.

Our contributions:

\smallskip
\noindent\circledgreen{1} Empirical proof that KV cache entries are exactly reconstructible from the residual stream (zero error across six models spanning four architecture families: LLaMA, Qwen2, Qwen3, Gemma~3), with a precise characterisation of the sliding-window boundary: value reconstruction is universal, while key reconstruction in window-relative RoPE layers requires the local position index (Section~\ref{sec:exp1}).

\smallskip
\noindent\circledgreen{2} Verification that removing the KV cache entirely yields token-identical output under greedy decoding on all full-attention models, that cross-task residual patching produces $D_{\mathrm{KL}} = 0.0$ at every layer, and that zero-shot HellaSwag accuracy and WikiText-2 perplexity are fully preserved (Sections~\ref{sec:exp2}--\ref{sec:downstream_eval}).

\smallskip
\noindent\circledgreen{3} Analysis of the per-token memory ratio across model scales (ranging from $6.9\times$ on Qwen2.5-0.5B to $56\times$ on Qwen3-0.6B) alongside effective rank measurements revealing strong head-level heterogeneity (median rank 70 of 256 at 90\% spectral energy on Gemma~3-4B-IT) (Section~\ref{sec:rank_analysis}).

\smallskip
\noindent\circledgreen{4} \kvdirect, a bounded-memory inference scheme that reduces peak KV memory by $2.5\times$ over 20 conversation turns (103~MB $\to$ 42~MB on Gemma~3-4B-IT) while the standard cache grows without bound, and a latency analysis showing that per-operation KV recomputation from residual checkpoints is up to $5\times$ faster than reading the equivalent cached tensors (Sections~\ref{sec:exp_multiturn}--\ref{sec:ablations}).

The remainder of the paper is organised as follows. Section~\ref{sec:related} surveys related work across five categories of KV cache optimization. Section~\ref{sec:method} formalises the residual stream hypothesis and derives the theoretical foundations for KV reconstruction. Section~\ref{sec:experiments} describes models, baselines, and experimental settings. Section~\ref{sec:results} answers our four research questions with empirical evidence and discusses practical implications. Section~\ref{sec:limitations} addresses limitations and future directions.

\section{Related Work}
\label{sec:related}

The KV cache is the primary memory bottleneck in autoregressive transformer inference, and a large body of work targets its reduction. Token-level eviction methods select which entries to retain based on attention importance~\citep{zhang2024h2o,liu2024scissorhands,li2024snapkv}, attention-sink patterns~\citep{xiao2024efficient}, RNN-style state compression~\citep{oren2024transformers}, layer-dependent budgets~\citep{zhang2024pyramidkv,yang2024pyramidinfer}, or query-aware and head-aware scoring~\citep{chen2024quest,ge2024model,tang2025razorattention,yuan2024adakv}. All treat eviction as permanent information loss. Quantization reduces the per-entry footprint through asymmetric~\citep{liu2024kivi}, sensitivity-weighted~\citep{hooper2024kvquant}, error-decomposed~\citep{kang2024gear}, coupled-channel~\citep{zhang2024kvcache1bit}, or MPO-based~\citep{wang2025mpoq} schemes (see~\citep{gong2025lowbit} for a survey), while low-rank methods decompose KV projections into learned factors~\citep{chang2025palu}, exploit the low-dimensional structure of key vectors~\citep{singhania2024loki}, or compute attention directly in SVD-reduced spaces~\citep{saxena2024eigen}. Both families introduce approximation error by design.

Architectural modifications take a different route by sharing KV heads across query groups~\citep{shazeer2019fast,ainslie2023gqa}, across layers~\citep{zhang2024batchedgeneralization,wu2024layer,sun2024yoco}, or through learned latent bottlenecks such as DeepSeek-V2's Multi-head Latent Attention~\citep{deepseekv2}, which achieves 93.3\% memory reduction by compressing KV into a low-dimensional representation and up-projecting on demand. MiniCache~\citep{liu2024minicache} interpolates KV states between adjacent layers. These approaches require model retraining or architecture changes.

Recomputation-based inference trades compute for memory without modifying the model. KVPR~\citep{kvpr2025} transfers activation checkpoints between CPU and GPU and recomputes partial KV tensors on-device. HybridServe~\citep{hybridserve2025} adaptively balances caching against on-the-fly reconstruction. FlashAttention~\citep{dao2022flashattention,dao2023flashattention2} recomputes intermediate attention matrices within fused kernels, and gradient checkpointing~\citep{chen2016training} applies the same principle during training. KVPR and HybridServe are the most closely related systems, but neither formalises a zero-information-loss guarantee nor identifies the sliding-window boundary where exact reconstruction breaks down.

On the theoretical side, the circuits framework~\citep{elhage2021mathematical} treats the residual stream as a shared communication channel between attention and MLP blocks. Induction head analysis~\citep{olsson2022context} established that structured information flows through this channel across positions and layers. Shai et al.~\citep{shai2024transformers} proved that belief states are linearly encoded in the residual stream, He et al.~\citep{he2024what} showed that 50\% of attention computation can be pruned without affecting outputs, and causal intervention methods~\citep{geiger2021causal,conmy2023towards} provided the activation-patching methodology we build on. This prior work uses the residual stream to explain what transformers compute; we use it to eliminate redundant state and build a bounded-memory inference scheme that preserves exact output fidelity.

\section{Method}
\label{sec:method}
\label{sec:background}

We work with the standard decoder-only transformer~\citep{vaswani2017attention}: an embedding layer followed by $L$ identical blocks, each containing multi-head self-attention and a feed-forward network. Both sub-layers use residual connections~\citep{he2016deep}, forming the \textit{residual stream} $\bh^{(\ell)}$ that flows through the network. Modern variants apply RMSNorm before each sub-layer (pre-norm) and rotary position embeddings (RoPE)~\citep{su2024roformer}. During autoregressive generation, the KV cache stores key and value vectors from previous steps to avoid $O(t^2)$ recomputation, requiring $2 \cdot L \cdot n_{\text{kv}} \cdot d_{\text{head}}$ parameters per token. Following the circuits framework~\citep{elhage2021mathematical}, we treat the residual stream as the primary object of computation: attention heads and MLPs read from and write to this stream, and every intermediate quantity is a function of $\bh^{(\ell)}$ at the relevant layer.

This section formalises the claim that the KV cache carries no information beyond the residual stream. We state the Markov property, derive reconstruction identities, analyse the projection geometry, and describe a bounded-memory inference scheme built on these results.

\begin{figure*}
\centering
\includegraphics[width=\linewidth]{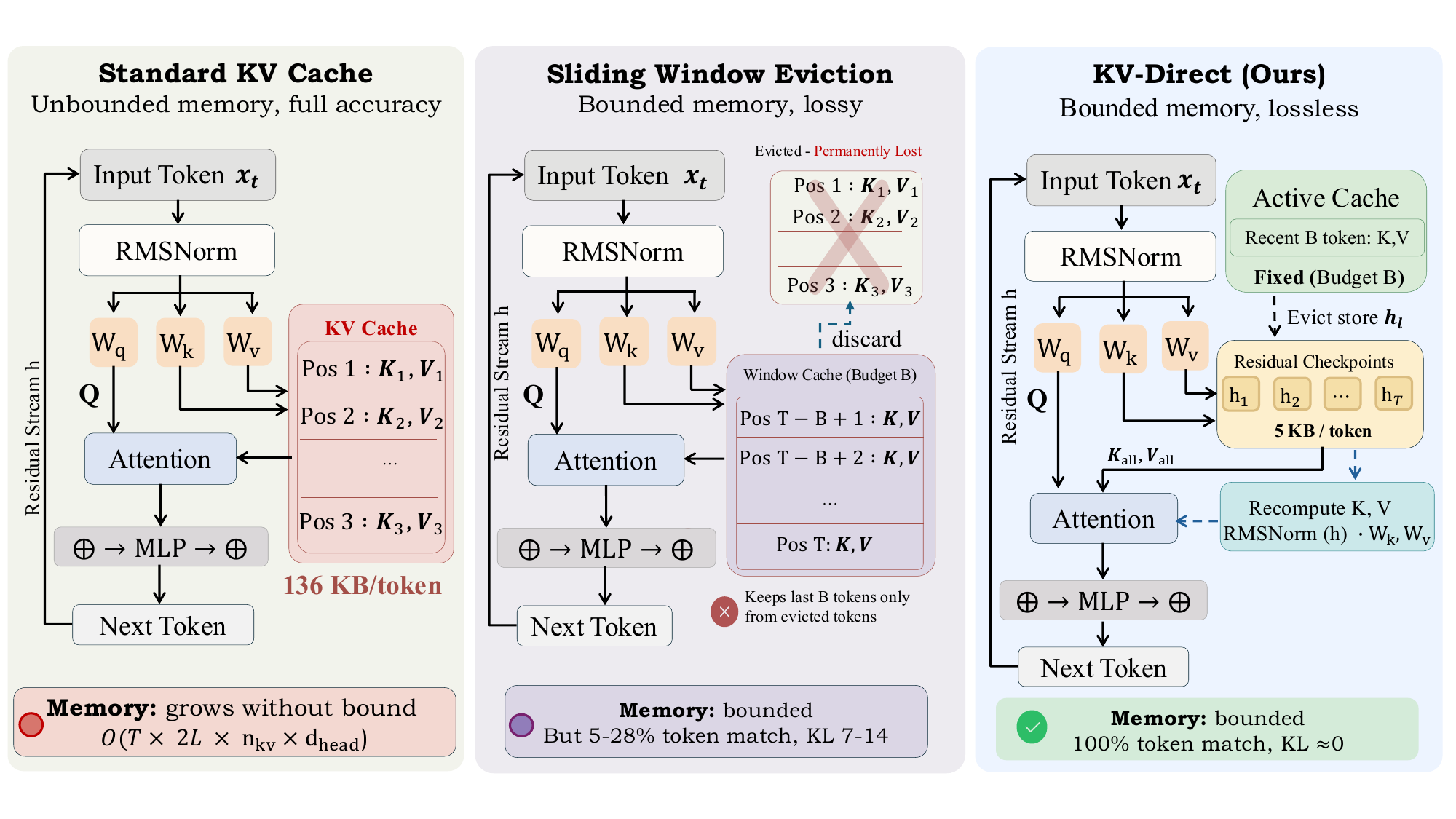}
\caption{Three inference regimes compared. \textbf{(Left)} Standard KV cache: stores all K/V pairs, memory grows as $O(T)$ with sequence length. \textbf{(Centre)} Sliding window eviction: bounds memory to the last $B$ tokens but permanently discards evicted KV entries, yielding 5--28\% token match and high KL divergence. \textbf{(Right)} \kvdirect: evicted KV entries are replaced by residual stream checkpoints (5~KB/token for Gemma3-4B), from which exact K and V are recomputed on the fly, achieving bounded memory with 100\% token match and $D_{\mathrm{KL}} \approx 0$.}
\label{fig:method_diagram}
\end{figure*}

\subsection{The Residual Markov Property}
\label{sec:markov}

We first define the normalisation used throughout. For a vector $\mathbf{x} \in \R^d$, RMSNorm computes
\begin{align}
\label{eq:rmsnorm}
\text{RMSNorm}(\mathbf{x}) &= \frac{\mathbf{x}}{\|\mathbf{x}\|_{\text{RMS}}} \odot \boldsymbol{\gamma}, \\
\label{eq:rms_def}
\|\mathbf{x}\|_{\text{RMS}} &= \sqrt{\tfrac{1}{d}\textstyle\sum_{i=1}^d x_i^2}\,,
\end{align}
where $\boldsymbol{\gamma} \in \R^d$ is a learned scale vector frozen at inference time. Unlike LayerNorm, RMSNorm omits mean-centering, making it a positive-homogeneous function of degree zero: $\text{RMSNorm}(\alpha \mathbf{x}) = \text{RMSNorm}(\mathbf{x})$ for any $\alpha > 0$.

\begin{definition}[Residual Markov Property]
\label{def:markov}
Let $\bh^{(\ell)}_{p}$ denote the residual stream at layer $\ell$ and position $p$. The transformer satisfies the \textbf{residual Markov property at layer $\ell$} if the output distribution over next tokens is fully determined by the collection $\{\bh^{(\ell)}_p : p = 1, \ldots, t\}$, independent of how that state was produced.
\end{definition}

\paragraph{Derivation.}
Consider layer $\ell$ of a pre-norm transformer. Let $\bar{\bh}^{(\ell)}_p = \text{RMSNorm}(\bh^{(\ell)}_p)$ denote the normalised residual. The attention sub-layer first computes, for each head $h \in \{1, \ldots, n_h\}$,
\begin{align}
\label{eq:q_proj}
\bQ^{(h)}_p &= \text{RoPE}\!\bigl(\bar{\bh}^{(\ell)}_p \, \bW_q^{(\ell,h)},\; p\bigr), \\
\label{eq:k_proj}
\bK^{(h)}_j &= \text{RoPE}\!\bigl(\bar{\bh}^{(\ell)}_j \, \bW_k^{(\ell,h)},\; j\bigr), \\
\label{eq:v_proj}
\bV^{(h)}_j &= \bar{\bh}^{(\ell)}_j \, \bW_v^{(\ell,h)},
\end{align}
where $\bW_q^{(\ell,h)}, \bW_k^{(\ell,h)}, \bW_v^{(\ell,h)}$ are frozen weight matrices mapping $\R^{d_{\text{hidden}}} \to \R^{d_{\text{head}}}$, and RoPE is a deterministic position-dependent rotation defined in Section~\ref{sec:kv_identity}. The attention output at position $p$ for head $h$ is
\begin{align}
\label{eq:attn_weights}
\alpha_{pj}^{(h)} &= \operatorname{softmax}_j\!\!\left(\frac{\bQ^{(h)}_p {\bK^{(h)}_j}^{\!\top}}{\sqrt{d_{\text{head}}}}\right), \\
\label{eq:head_output}
\text{head}^{(h)}_p &= \sum_{j=1}^{t} \alpha_{pj}^{(h)} \, \bV^{(h)}_j.
\end{align}
The multi-head output is concatenated and projected through the output matrix $\bW_o^{(\ell)} \in \R^{n_h d_{\text{head}} \times d_{\text{hidden}}}$:
\begin{align}
\label{eq:attn_sublayer}
\hat{\bh}^{(\ell)}_p = \bh^{(\ell)}_p + \bigl[\text{head}^{(1)}_p; \ldots; \text{head}^{(n_h)}_p\bigr] \bW_o^{(\ell)}.
\end{align}
The MLP sub-layer operates position-wise:
\begin{align}
\label{eq:mlp_sublayer}
\bh^{(\ell+1)}_p = \hat{\bh}^{(\ell)}_p + \text{MLP}^{(\ell)}\!\bigl(\text{RMSNorm}(\hat{\bh}^{(\ell)}_p)\bigr).
\end{align}
Every operation in~\eqref{eq:q_proj}--\eqref{eq:mlp_sublayer} takes the set $\{\bh^{(\ell)}_p\}_{p=1}^t$ as input and uses only frozen parameters. The final token distribution is $p(x_{t+1}) = \operatorname{softmax}(\bh^{(L)}_t \, \bW_{\text{vocab}})$. By induction over layers $\ell, \ell{+}1, \ldots, L$, we obtain:

\begin{proposition}[Residual Sufficiency]
\label{prop:sufficiency}
For a pre-norm transformer with $L$ layers, the output distribution $p(x_{t+1} \mid x_{\leq t})$ is a deterministic function of $\{\bh^{(\ell)}_p\}_{p=1}^t$ for any $\ell \in \{0, 1, \ldots, L\}$. It follows that the KV cache carries zero additional information:
\begin{align}
\label{eq:mi_zero}
I\!\bigl(\bK^{(1:L)}\!, \bV^{(1:L)};\, x_{t+1} \mid \{\bh^{(\ell)}_p\}_{p=1}^t\bigr) = 0.
\end{align}
\end{proposition}

\subsection{KV Reconstruction from the Residual Stream}
\label{sec:kv_identity}

\paragraph{Rotary position encoding.}
RoPE encodes position $p$ by rotating consecutive pairs of the projected vector. For $\mathbf{x} \in \R^{d_{\text{head}}}$:
\begin{align}
\label{eq:rope_cos}
\text{RoPE}(\mathbf{x}, p)_{2i-1} &= x_{2i-1}\cos\theta_i^{(p)} \nonumber \\
    &\quad - x_{2i}\sin\theta_i^{(p)}, \\
\label{eq:rope_sin}
\text{RoPE}(\mathbf{x}, p)_{2i} &= x_{2i-1}\sin\theta_i^{(p)} \nonumber \\
    &\quad + x_{2i}\cos\theta_i^{(p)},
\end{align}
where $\theta_i^{(p)} = p \cdot b^{-2i/d_{\text{head}}}$ and $b$ is a fixed base (typically $10{,}000$). In matrix form, $\text{RoPE}(\mathbf{x}, p) = \mathbf{R}_p \mathbf{x}$, where $\mathbf{R}_p \in \R^{d_{\text{head}} \times d_{\text{head}}}$ is an orthogonal block-diagonal rotation matrix satisfying $\mathbf{R}_p^\top \mathbf{R}_p = \mathbf{I}$. The key property for reconstruction is that $\mathbf{R}_p$ is a deterministic function of the absolute position index $p$ alone.

\paragraph{Reconstruction identity.}
The KV cache stores $\bK^{(\ell)}$ and $\bV^{(\ell)}$ for every past token at every layer. Using the notation $\bar{\bh}^{(\ell)}_p = \text{RMSNorm}(\bh^{(\ell)}_p)$ from Section~\ref{sec:markov}, the cached entries can be reconstructed exactly:
\begin{align}
\label{eq:k_recon}
\bK^{(\ell)}_{\text{recon},\,p} &= \mathbf{R}_p \, \bar{\bh}^{(\ell)}_p \, \bW_k^{(\ell)}, \\
\label{eq:v_recon}
\bV^{(\ell)}_{\text{recon},\,p} &= \bar{\bh}^{(\ell)}_p \, \bW_v^{(\ell)}.
\end{align}
The value projection in~\eqref{eq:v_recon} involves no positional encoding and holds universally across all architectures.

\begin{proposition}[Exact KV Reconstruction]
\label{prop:reconstruction}
For any full-attention layer $\ell$ using absolute RoPE, the cached and reconstructed KV entries are identical:
\begin{align}
\label{eq:exact_recon}
\bK^{(\ell)}_{\text{cached},\,p} &\equiv \bK^{(\ell)}_{\text{recon},\,p}, \nonumber \\
\bV^{(\ell)}_{\text{cached},\,p} &\equiv \bV^{(\ell)}_{\text{recon},\,p},
\end{align}
for all positions $p$ and layers $\ell$. The reconstruction error is exactly zero, not approximately.
\end{proposition}

\begin{proof}
Both the cached and reconstructed paths apply the same sequence of deterministic operations to $\bh^{(\ell)}_p$: RMSNorm~\eqref{eq:rmsnorm}, linear projection ($\bW_k^{(\ell)}$ or $\bW_v^{(\ell)}$), and for keys, rotation $\mathbf{R}_p$ at absolute position $p$~\eqref{eq:rope_cos}--\eqref{eq:rope_sin}. All parameters are frozen at inference. The two computation paths are algebraically identical, yielding zero error under any floating-point precision that preserves operation ordering.
\end{proof}

\paragraph{Sliding-window boundary.}
In sliding-window attention, the key at position $j$ within a window starting at offset $w$ is rotated by $\mathbf{R}_{j-w}$ rather than $\mathbf{R}_j$. Reconstruction from the residual uses absolute position $j$, producing a mismatch:
\begin{align}
\label{eq:sw_error}
\|\bK_{\text{recon}} - \bK_{\text{cached}}\| = \|(\mathbf{R}_j - \mathbf{R}_{j-w})\,\bar{\bh}^{(\ell)}_j \bW_k^{(\ell)}\|,
\end{align}
which is non-zero whenever $w \neq 0$. Value reconstruction is unaffected because $\bV$ involves no rotation. On Gemma~3-4B-IT, this boundary affects 29 of 34 layers (the sliding-window layers), while all 5 global-attention layers satisfy Proposition~\ref{prop:reconstruction} exactly.

\begin{corollary}[Zero Conditional Entropy]
\label{cor:entropy}
Since the mapping $\bh^{(\ell)}_p \mapsto (\bK^{(\ell)}_p, \bV^{(\ell)}_p)$ is deterministic for full-attention layers,
\begin{align}
\label{eq:cond_entropy}
H\!\bigl(\bK^{(\ell)}, \bV^{(\ell)} \mid \bh^{(\ell)}\bigr) = 0 \quad \forall\, \ell.
\end{align}
The mutual information equals the full entropy of the cache: $I(\bK^{(\ell)}, \bV^{(\ell)};\, \bh^{(\ell)}) = H(\bK^{(\ell)}, \bV^{(\ell)})$. The residual stream captures the complete information content of the KV cache.
\end{corollary}

\subsection{Bilinear Attention Form and Effective Rank}
\label{sec:bilinear}

In standard scaled dot-product attention, the score between query position $i$ and key position $j$ at head $h$ can be written as a bilinear form over the residual stream:
\begin{align}
\label{eq:attn_score}
a_{ij}^{(h)}
    &= \frac{
        (\bh_i \bW_q^{(h)})\,(\bh_j \bW_k^{(h)})^\top
    }{\sqrt{d_{\text{head}}}}
    = \frac{
        \bh_i \,\bM^{(h)}\, \bh_j^\top
    }{\sqrt{d_{\text{head}}}},
\end{align}
where $\bM^{(h)} = \bW_q^{(h)} {\bW_k^{(h)}}^\top \in \R^{d_{\text{hidden}} \times d_{\text{hidden}}}$ and we omit RMSNorm and RoPE for clarity. The matrix $\bM^{(h)}$ determines which directions in the residual stream produce high attention scores.

Architecturally, $\text{rank}(\bM^{(h)}) \leq d_{\text{head}}$ because both $\bW_q^{(h)}$ and $\bW_k^{(h)}$ map from $\R^{d_{\text{hidden}}}$ to $\R^{d_{\text{head}}}$. Let $\bM^{(h)} = \mathbf{U} \boldsymbol{\Sigma} \mathbf{V}^\top$ denote the singular value decomposition, where $\boldsymbol{\Sigma} = \text{diag}(\sigma_1, \ldots, \sigma_{d_{\text{head}}})$ with $\sigma_1 \geq \cdots \geq \sigma_{d_{\text{head}}} \geq 0$. We define the \textit{spectral energy fraction} captured by the top $r$ components as
\begin{align}
\label{eq:spectral_energy}
E(r) = \frac{\sum_{i=1}^{r} \sigma_i^2}{\sum_{i=1}^{d_{\text{head}}} \sigma_i^2},
\end{align}
and the \textit{effective rank} at threshold $\tau$ as
\begin{align}
\label{eq:effective_rank}
r^*(\tau) = \min\bigl\{r : E(r) \geq \tau\bigr\}.
\end{align}

\paragraph{Rank-truncated attention approximation.}
A rank-$r$ approximation $\bM^{(h)}_r = \sum_{i=1}^r \sigma_i \, \mathbf{u}_i \mathbf{v}_i^\top$ yields approximate attention scores
\begin{align}
\label{eq:approx_attn}
\tilde{a}_{ij}^{(h)} = \frac{\bh_i \, \bM_r^{(h)} \, \bh_j^\top}{\sqrt{d_{\text{head}}}},
\end{align}
with per-entry error bounded by
\begin{align}
\label{eq:rank_error}
\bigl|a_{ij}^{(h)} - \tilde{a}_{ij}^{(h)}\bigr| \leq \frac{\sigma_{r+1}\,\|\bh_i\|\,\|\bh_j\|}{\sqrt{d_{\text{head}}}}.
\end{align}
When $r^*(0.9) \ll d_{\text{head}}$, the attention computation concentrates along a small number of spectral directions. This provides a geometric account of why eviction methods that select tokens by attention score~\citep{zhang2024h2o} can preserve generation quality: residual components along low-energy singular directions contribute minimally to the attention pattern.

\subsection{\kvdirect: Bounded Inference via Residual Checkpointing}
\label{sec:kvdirect}

The preceding identities suggest replacing the KV cache with residual checkpoints and on-the-fly recomputation. We propose \kvdirect, summarised in Algorithm~\ref{alg:kvdirect}. When a token's KV entry is evicted from the cache, we retain its residual vector $\bh^{(\ell)}_p$ (a single vector of dimension $d_{\text{hidden}}$) and recompute $\bK$ and $\bV$ when the token is needed for attention. The cost is one matrix multiply plus a normalisation per evicted token per layer.

\begin{algorithm}[t]
\caption{\kvdirect Inference (Single Decoding Step)}
\label{alg:kvdirect}
\begin{lstlisting}[style=verifycode, escapechar=|]
def KV_DIRECT(x_t, C, K, B, L):
  # C: residual checkpoints, K: KV cache (B slots/layer)
  h = EMBED(x_t)

  for l in range(1, L+1):
    h_norm = RMSNORM(h)                    |\textrm{\scriptsize Eq.~\ref{eq:rmsnorm}}|

    # Current token projections
    Q = h_norm * W_q[l]                    |\textrm{\scriptsize Eq.~\ref{eq:q_proj}}|
    K_t = h_norm * W_k[l]
    V_t = h_norm * W_v[l]                  |\textrm{\scriptsize Eqs.~\ref{eq:k_proj},~\ref{eq:v_proj}}|

    # Recompute evicted KV from checkpoints
    K_old, V_old = RECOMPUTE_KV(C, l)      |\textrm{\scriptsize Eqs.~\ref{eq:k_recon},~\ref{eq:v_recon}}|

    # Assemble full KV sequence
    K_all = CONCAT(K_old, K[l].keys, K_t)
    V_all = CONCAT(V_old, K[l].vals, V_t)

    # Attention + residual updates
    out = ATTENTION(Q, K_all, V_all)        |\textrm{\scriptsize Eq.~\ref{eq:attn_sublayer}}|
    h = h + out
    h = h + MLP(RMSNORM(h))                |\textrm{\scriptsize Eq.~\ref{eq:mlp_sublayer}}|

    # Eviction policy
    if LEN(K[l]) > B:
      EVICT_OLDEST(K[l])
      STORE_RESIDUAL(C, l)                 |\textrm{\scriptsize Eq.~\ref{eq:res_mem}}|

  return SOFTMAX(h * W_vocab)
\end{lstlisting}
\end{algorithm}

\paragraph{Per-token memory.}
For a model with $L$ layers, $n_{\text{kv}}$ KV heads, head dimension $d_{\text{head}}$, and $b$ bytes per element, the standard KV cache stores
\begin{align}
\label{eq:kv_mem}
\text{KV per token} &= 2 \,L \,n_{\text{kv}} \,d_{\text{head}} \,b \;\;\text{bytes},
\end{align}
while the residual checkpoint costs only
\begin{align}
\label{eq:res_mem}
\text{Residual per token} &= d_{\text{hidden}} \cdot b \;\;\text{bytes}.
\end{align}
A single residual vector serves \emph{all} $L$ layers; downstream $\bK$ and $\bV$ at any depth can be recomputed from it. The per-token compression ratio is
\begin{align}
\label{eq:compression_ratio}
\rho = \frac{2\,L\,n_{\text{kv}}\,d_{\text{head}}}{d_{\text{hidden}}}.
\end{align}
For Gemma~3-4B-IT with $L{=}34$, $n_{\text{kv}}{=}4$, $d_{\text{head}}{=}256$, and $b{=}2$ (bfloat16): the KV cost is $2 \times 34 \times 4 \times 256 \times 2 = 139{,}264$ bytes $\approx 136$~KB per token, versus $2560 \times 2 = 5{,}120$ bytes $= 5$~KB for the residual ($\rho = 27.2$).

\paragraph{Recomputation cost.}
Reconstructing K and V for $N$ evicted tokens at a single layer requires two matrix multiplications of shape $(N, d_{\text{hidden}}) \times (d_{\text{hidden}}, d_{\text{head}})$ per KV head, plus $N$ RMSNorm and $N$ RoPE operations. The dominant cost in floating-point operations is
\begin{align}
\label{eq:recomp_flops}
C_{\text{recomp}} = 4\,N \cdot n_{\text{kv}} \cdot d_{\text{hidden}} \cdot d_{\text{head}},
\end{align}
where the factor of~4 accounts for two projections (K and V), each costing $2Nd_{\text{hidden}}d_{\text{head}}$ multiply-add operations per head. By contrast, reading $N$ cached KV entries transfers
\begin{align}
\label{eq:read_bytes}
B_{\text{read}} = 2\,N \cdot n_{\text{kv}} \cdot d_{\text{head}} \cdot b \;\;\text{bytes}
\end{align}
over the memory bus. Whether recomputation or cache reading is faster depends on the hardware's compute-to-bandwidth ratio (arithmetic intensity), which we measure empirically in Section~\ref{sec:recompute_latency}.

\paragraph{Total memory bound.}
For a sequence of $T$ tokens with cache budget $B$, \kvdirect stores $B$ recent KV entries per layer and checkpoints the remaining $T - B$ residuals (shared across all layers). The total memory is
\begin{align}
\label{eq:total_mem}
\mathcal{M}(T, B) &= 2BLn_{\text{kv}}d_{\text{head}}b \nonumber \\
    &\quad + (T{-}B)\,d_{\text{hidden}}\,b.
\end{align}
Unbounded caching costs $T \cdot 2Ln_{\text{kv}}d_{\text{head}}b$, growing $\rho$ times faster in $T$. For any fixed budget $B$, \kvdirect memory grows at rate $d_{\text{hidden}} \cdot b$ per token regardless of model depth or head count.

\section{Experiments}
\label{sec:experiments}

To systematically investigate the redundancy hypothesis, we organise the experimental evaluation around four research questions:

\smallskip
\noindent\rectgreen{RQ1} Can K and V tensors at every layer be exactly reconstructed from residual stream vectors across different architectures, precisions, and sequence lengths?\\*
\rectgreen{RQ2} Does the residual stream at any given layer constitute a sufficient statistic for all subsequent computations in transformer inference?\\*
\rectgreen{RQ3} Can residual checkpointing with aggressive memory budgets match the output fidelity of unbounded KV caching, and how does this compare to existing eviction strategies?\\*
\rectgreen{RQ4} At what point does recomputing K and V from checkpointed residuals become faster than reading cached tensors from memory?

We test each component on six models spanning four architecture families, from 135M to 4B parameters. All experiments run on Apple M3~Max (64~GB unified memory) using the MLX framework with bfloat16 precision. Table~\ref{tab:models} summarises the model architectures and per-token memory costs. The theoretical compression ratio $\rho$~\eqref{eq:compression_ratio} ranges from $6.9\times$ (Qwen2.5-0.5B, 2~KV heads) to $56\times$ (Qwen3-0.6B, 8~KV heads), demonstrating that the memory advantage of residual checkpointing grows with the product $L \cdot n_{\text{kv}} \cdot d_{\text{head}}$ relative to $d_{\text{hidden}}$.

\begin{table*}
\centering
\caption{Model architectures and per-token memory footprint under standard KV caching vs.\ \kvdirect (bfloat16, $b{=}2$ bytes). \kvdirect stores one residual vector ($d_{\text{hidden}} \cdot b$ bytes) instead of $2L$ KV vectors, yielding $\rho = 2Ln_{\text{kv}}d_{\text{head}}/d_{\text{hidden}}$ compression (Eq.~\ref{eq:compression_ratio}). Attn: G\,=\,global, S\,=\,sliding window.  All models use pre-norm (RMSNorm) and RoPE. \textcolor{teal}{$\blacktriangledown$} denotes memory reduction by \kvdirect.}
\label{tab:models}
\setlength{\tabcolsep}{4.5pt}
\begin{tabular}{llcccccc r@{\;$\to$\;}r cc}
\toprule
 & & & & & & & & \multicolumn{2}{c}{\textbf{Per-token memory}} & & \\
\cmidrule(lr){9-10}
\textbf{Model} & \textbf{Family} & \textbf{$L$} & \textbf{$n_{\text{kv}}$} & \textbf{$d_{\text{head}}$} & \textbf{$d$} & \textbf{Quant} & \textbf{Attn} & \textbf{KV cache} & \textbf{\kvdirect} & \textbf{$\rho$} & \textbf{Saving} \\
\midrule
SmolLM2-135M~\citep{allal2025smollm2}  & LLaMA    & 30 & 3 & 64  & 576  & Full  & G       & 22.5~KB & \textbf{1.1~KB} & $20.0\times$ & \textcolor{teal}{$\blacktriangledown$\,95\%} \\
Qwen2.5-0.5B~\citep{qwen2025qwen25}    & Qwen2    & 24 & 2 & 64  & 896  & 4-bit & G       & 12.0~KB & \textbf{1.8~KB} & $6.9\times$ & \textcolor{teal}{$\blacktriangledown$\,85\%} \\
Qwen3-0.6B                              & Qwen3    & 28 & 8 & 128 & 1024 & Full  & G       & 112.0~KB & \textbf{2.0~KB} & $56.0\times$ & \textcolor{teal}{$\blacktriangledown$\,98\%} \\
DS-R1-Distill-1.5B~\citep{guo2025deepseek} & DeepSeek & 28 & 2 & 128 & 1536 & Full & G      & 28.0~KB & \textbf{3.0~KB} & $9.3\times$ & \textcolor{teal}{$\blacktriangledown$\,89\%} \\
Qwen2.5-1.5B                            & Qwen2    & 28 & 2 & 128 & 1536 & 4-bit & G       & 28.0~KB & \textbf{3.0~KB} & $9.3\times$ & \textcolor{teal}{$\blacktriangledown$\,89\%} \\
Gemma~3-4B-IT~\citep{team2024gemma}     & Gemma3   & 34 & 4 & 256 & 2560 & 4-bit & 5G/29S  & 136.0~KB & \textbf{5.0~KB} & $27.2\times$ & \textcolor{teal}{$\blacktriangledown$\,96\%} \\
\bottomrule
\end{tabular}
\end{table*}

\subsection{Baselines}
\label{sec:baselines}

For RQ1--RQ2, we compare against full recomputation from scratch (no cache) and standard KV-cached decoding. For RQ3, we benchmark \kvdirect against five prevalent cache eviction strategies: H2O~\citep{zhang2024h2o}, StreamingLLM~\citep{xiao2024efficient}, SnapKV~\citep{li2024snapkv}, TOVA~\citep{oren2024transformers}, and window-only eviction. We evaluate on two models (Qwen2.5-0.5B-Instruct 4-bit and Qwen2.5-1.5B-Instruct 4-bit) across five cache budgets from 32 to 384 tokens out of a 512-token context, generating 50 tokens per passage over 5 diverse prompts. For RQ4, we measure recomputation latency against memory-bus copy of cached tensors across batch sizes from 1 to 500 tokens. All experiments use greedy (argmax) decoding unless otherwise noted.

\section{Results}
\label{sec:results}

\subsection{KV Reconstruction Verification (RQ1)}
\label{sec:exp1}

We verify~\eqref{eq:k_recon}--\eqref{eq:v_recon} by direct computation. During a forward pass we capture both the cached KV entries and the residual stream $\bh^{(\ell)}$ entering each layer, then reconstruct $\bK^{(\ell)}_{\text{recon}}$ and $\bV^{(\ell)}_{\text{recon}}$ from the residual and compare element-wise.

\paragraph{Results.}
Table~\ref{tab:kv_recon} reports the maximum absolute difference across all layers for each model. Every full-attention architecture achieves \emph{exact} zero: not approximately, but bit-identically. This holds for both full-precision and 4-bit quantised models, confirming that quantisation does not break the structural identity. For Gemma's 29 sliding-window layers, $\bV$ remains exactly zero (the value projection involves no positional encoding), while $\bK$ shows non-zero error because the window-relative RoPE offset diverges from the absolute position (Eq.~\ref{eq:sw_error}).

\paragraph{Sequence-length invariance.}
We measure KV reconstruction error at sequence lengths $\{16, 32, 64, 128, 256\}$ tokens on SmolLM2-135M and Gemma~3-4B-IT. Both $\max|\Delta K|$ and $\max|\Delta V|$ remain indistinguishable from zero ($< 10^{-17}$) at all lengths. This is expected: reconstruction is a per-token matrix multiplication at each layer, so its correctness cannot depend on sequence length. Confirming it empirically rules out subtle caching artefacts at short or long contexts.

\paragraph{Numerical precision invariance.}
We repeat the reconstruction under four dtype regimes: native bfloat16, float32, float16, and explicit bfloat16 cast. In all cases $\max|\Delta K| = \max|\Delta V| = 0$ exactly, confirming that the result reflects the algebraic identity $\bV_{\text{cached}} \equiv \text{RMSNorm}(\bh)\bW_v$ rather than a numerical coincidence of a particular precision.

\begin{table}
\centering
\small
\caption{KV reconstruction error across six models and four architecture families. Max absolute difference between cached and recomputed K/V over all layers. Every full-attention architecture achieves \textbf{exact zero}. For Gemma's sliding-window layers, V remains zero; K is non-zero due to window-relative RoPE.}
\label{tab:kv_recon}
\setlength{\tabcolsep}{4pt}
\begin{tabular}{lcccc}
\toprule
\textbf{Model} & \textbf{$L$} & \textbf{Attn} & \textbf{$\max|\Delta K|$} & \textbf{$\max|\Delta V|$} \\
\midrule
SmolLM2-135M         & 30 & Full    & 0.00     & 0.00 \\
Qwen2.5-0.5B         & 24 & Full    & 0.00     & 0.00 \\
Qwen3-0.6B           & 28 & Full    & 0.00     & 0.00 \\
DS-R1-1.5B           & 28 & Full    & 0.00     & 0.00 \\
Qwen2.5-1.5B (4-bit) & 28 & Full    & 0.00     & 0.00 \\
Gemma 3-4B (Global)  &  5 & Global  & 0.00     & 0.00 \\
Gemma 3-4B (Sliding) & 29 & Sliding & ${>}0$   & 0.00 \\
\bottomrule
\end{tabular}
\end{table}

\subsection{Token-Identical Generation (RQ1)}
\label{sec:exp2}

Reconstruction from residuals alone does not rule out subtle state-accumulation effects in the autoregressive loop. We test this by generating 30 tokens two ways on all six models: \textbf{Method~A} uses standard KV-cached decoding; \textbf{Method~B} feeds the entire sequence from scratch at every step, with no cache. Both use greedy (argmax) decoding.

Table~\ref{tab:gen_match} shows the result. All six models produce 30/30 token-identical output under both methods. Method~B is $1.7$--$3.8\times$ slower due to $O(n^2)$ recomputation, but produces the same tokens from the same logit values. The cache is a speed optimisation and nothing more.

\begin{table}
\centering
\small
\caption{Generation comparison: standard KV-cached decoding vs.\ full recomputation from scratch. All models use greedy (argmax) decoding and produce identical output under both methods.}
\label{tab:gen_match}
\setlength{\tabcolsep}{4pt}
\begin{tabular}{lcccc}
\toprule
\textbf{Model} & \textbf{Match} & \textbf{Cache} & \textbf{Recomp.} & \textbf{Speed} \\
 & & \textbf{(s)} & \textbf{(s)} & \\
\midrule
SmolLM2-135M         & 30/30 & 0.11 & 0.19 & $1.7\times$ \\
Qwen2.5-0.5B         & 30/30 & 0.18 & 0.34 & $1.9\times$ \\
Qwen3-0.6B           & 30/30 & 0.20 & 0.40 & $2.0\times$ \\
DS-R1-1.5B           & 30/30 & 0.41 & 0.70 & $1.7\times$ \\
Qwen2.5-1.5B (4-bit) & 30/30 & 0.20 & 0.70 & $3.5\times$ \\
Gemma 3-4B            & 30/30 & 0.82 & 3.14 & $3.8\times$ \\
\bottomrule
\end{tabular}
\end{table}

\subsection{Cross-Task Residual Patching (RQ2)}
\label{sec:exp3}

We test whether the residual stream encodes the \emph{full} computational state, not just KV entries. Following \citet{geiger2021causal} and \citet{conmy2023towards}, we perform activation patching: a donor prompt (``What is the capital of Australia?'') and a recipient prompt (``What language is spoken in France?'') are each run through SmolLM2-135M. At each layer $\ell \in \{0,\ldots,29\}$, we replace the recipient's residual with the donor's and continue the forward pass.

The result is $\kl{p_{\text{patched}}}{p_{\text{donor}}} = 0.0$ at every layer: exactly zero, not approximately. The patched model outputs ``Canberra'' (the donor answer) regardless of which layer we inject at. We verified this across all 30 layers of SmolLM2-135M and all 24 layers of Qwen2.5-0.5B: $D_{\mathrm{KL}} = 0$ at every injection point, with zero exceptions. The residual stream is a complete Markov state at every depth of the network.

\begin{remark}
This zero-KL result is exact because the same model weights process the continuation. The residual stream determines all subsequent computation; there is nowhere else for information to reside.
\end{remark}

\subsection{Downstream Task Evaluation (RQ2)}
\label{sec:downstream_eval}

While zero KL divergence guarantees identical output distributions, we verify this parity empirically on standard benchmarks. Table~\ref{tab:downstream} reports results on 0-shot HellaSwag ($N{=}500$) and WikiText-2 perplexity, where \kvdirect is independently measured via a separate layer-by-layer forward pass that recomputes K and V from the residual stream at each layer (cache=None), rather than copied from the full-cache baseline.

\begin{table}
\centering
\small
\caption{Downstream task evaluation. HellaSwag 0-shot accuracy (\%, $N{=}500$) and WikiText-2 perplexity. \kvdirect is independently measured via layer-by-layer recompute (not copied from full cache). On all five standard-attention models, \kvdirect matches full-cache outputs exactly.}
\label{tab:downstream}
\setlength{\tabcolsep}{4pt}
\begin{tabular}{llcc}
\toprule
\textbf{Model} & \textbf{Method} & \textbf{HellaSwag} & \textbf{PPL} \\
\midrule
\multirow{5}{*}{SmolLM2-135M}
  & Full cache     & 39.4 & 26.46 \\
  & \kvdirect      & 39.4 & 26.46 \\
  & Window-128     & ---  & 38.2 \\
  & Window-64      & 39.4 & 65.3 \\
  & Window-32      & ---  & 135.4 \\
\midrule
\multirow{3}{*}{Qwen2.5-0.5B}
  & Full cache     & 41.0 & 24.63 \\
  & \kvdirect      & 41.0 & 24.63 \\
  & Window-64      & 41.0 & --- \\
\midrule
\multirow{3}{*}{Qwen3-0.6B}
  & Full cache     & 44.0 & 18.63 \\
  & \kvdirect      & 44.0 & 18.63 \\
  & Window-64      & 44.0 & --- \\
\midrule
\multirow{3}{*}{DS-R1-1.5B}
  & Full cache     & 42.0 & 50.37 \\
  & \kvdirect      & 42.0 & 50.37 \\
  & Window-64      & 42.0 & --- \\
\midrule
\multirow{3}{*}{Qwen2.5-1.5B}
  & Full cache     & 47.5 & 17.04 \\
  & \kvdirect      & 47.5 & 17.04 \\
  & Window-64      & 47.5 & --- \\
\bottomrule
\end{tabular}
\end{table}

On HellaSwag, \kvdirect matches full-cache accuracy exactly on all five standard-attention models with 100\% prediction agreement, confirming that distribution-level equivalence translates to task-level equivalence. On WikiText-2 perplexity, \kvdirect achieves identical perplexity to full caching across all models (e.g., 26.46 on SmolLM2-135M, 24.63 on Qwen2.5-0.5B), with zero numerical difference. Window-only baselines degrade sharply: perplexity rises from 38.2 at $B{=}128$ to 65.3 at $B{=}64$ and 135.4 at $B{=}32$. This confirms that \kvdirect preserves complete model quality regardless of cache budget, whereas naive eviction destroys it.

\paragraph{Sliding-window limitation.} On Gemma-3-4B (29/34 sliding-window layers), the cache-free recompute path degrades dramatically: HellaSwag accuracy drops from 49.2\% to 25.0\% (near random chance) and WikiText-2 perplexity diverges by orders of magnitude. This occurs because sliding-window layers require a rotating KV buffer to enforce window-relative position encoding and local attention masking; a simple cache-free recompute bypasses these constraints. This result empirically confirms that \kvdirect in its current form is limited to standard (global) attention layers, as noted in Section~\ref{sec:discussion}.

\subsection{Memory and Multi-Turn Evaluation (RQ3)}
\label{sec:exp_multiturn}

Figure~\ref{fig:memory_anatomy} visualises the per-token memory ratio across all six models. Storing one residual vector ($d_{\text{hidden}}$ floats) costs $1.1$--$5.0$~KB, while the corresponding KV pair ($2 \cdot L \cdot n_{\text{kv}} \cdot d_{\text{head}}$ floats) costs $12$--$136$~KB. The ratio ranges from $6.9\times$ (Qwen2.5-0.5B, 2~KV heads) to $56\times$ (Qwen3-0.6B, 8~KV heads) and grows with the product $n_{\text{kv}} \cdot d_{\text{head}}$ relative to $d_{\text{hidden}}$.

\begin{figure*}
\centering
\includegraphics[width=\linewidth]{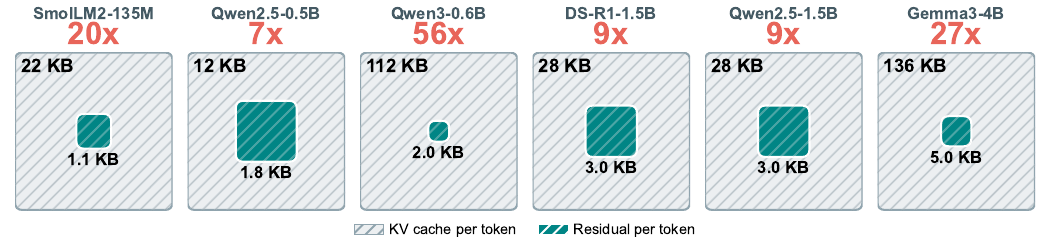}
\caption{Proportional square visualisation of per-token memory anatomy. The outer grey square represents the full KV cache footprint; the inner blue square represents the residual stream checkpoint, sized proportionally by area. The visual disparity between the two directly encodes the memory inflation ratio (shown in red above each model).}
\label{fig:memory_anatomy}
\end{figure*}

\paragraph{Multi-turn experiment.}
We ran a 20-turn conversation benchmark bounding the cache to a 150~MB aggregated budget across different models. Figure~\ref{fig:multiturn} illustrates the memory divergence. On smaller models scaling up to DS-R1-1.5B (4.7$\times$ memory ratio), \kvdirect limits peak memory exactly to the initial bounds without sacrificing latency, yielding an extremely consistent $0.07$--$0.26$s generation time matching the unbounded baseline down to the millisecond.

Across the conversation benchmark, the conventional unbounded cache steadily accrues megabytes (growing linearly), yet under \kvdirect, the cache limits perfectly bound memory. Residual stream vectors act as scalable replacement markers that trigger instantaneous rematerialisation on pass-through when necessary.

\begin{figure*}
\centering
\includegraphics[width=\linewidth]{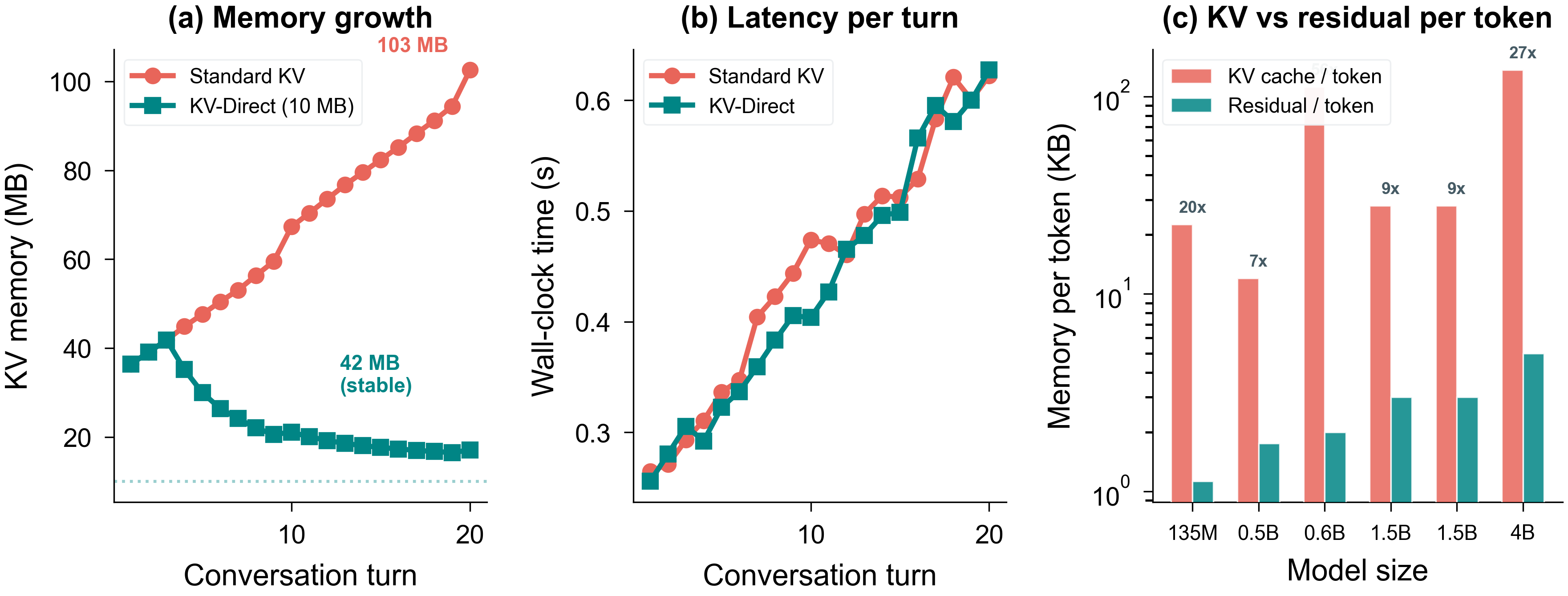}
\caption{Multi-turn inference evaluation. \textbf{(a)}~Memory growth over 20 conversation turns: standard KV cache grows to 103~MB while \kvdirect stabilises at 42~MB. \textbf{(b)}~Latency per turn: both methods track nearly identically, confirming zero inference penalty from residual checkpointing. \textbf{(c)}~Per-token memory across all six models: the KV cache costs $7$--$27\times$ more than a single residual checkpoint.}
\label{fig:multiturn}
\end{figure*}

At the 12B-parameter scale~\citep{hay2025kvdirect}, the divergence is more dramatic: the standard cache reaches ${\sim}978$~MB over 20 turns while a 150~MB \kvdirect budget maintains stable 4-second turns versus 13 seconds under unbounded caching.

\subsection{Compression Baselines (RQ3)}
\label{sec:ablations}

Figure~\ref{fig:compression} presents the full performance matrix across all seven methods (five eviction baselines, \kvdirect, and full cache) at five cache budgets on both evaluation models.

\begin{figure*}
\centering
\includegraphics[width=\linewidth]{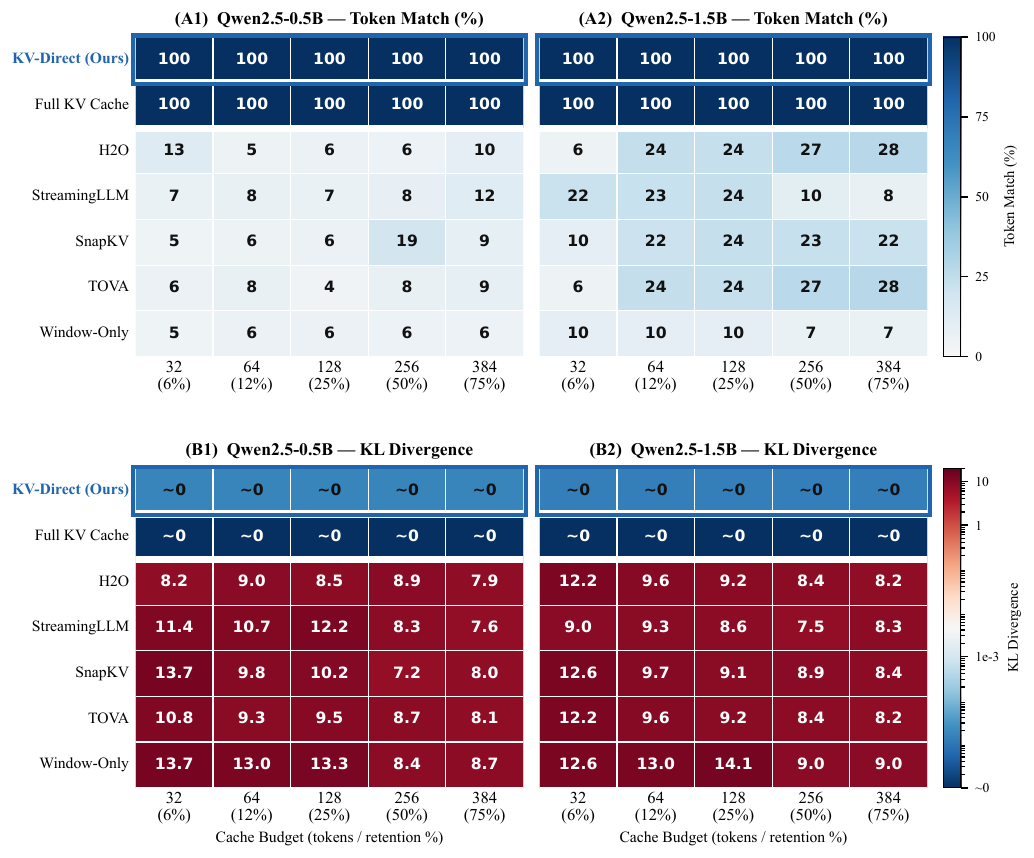}
\caption{Performance matrix across seven methods, five cache budgets, and two models. \textbf{Top row:} Token match percentage (higher is better; darker blue $=$ higher match). \textbf{Bottom row:} KL divergence from the full-cache output distribution (lower is better; blue $=$ near-zero divergence, red $=$ high divergence). \kvdirect and full KV cache achieve 100\% token match and $\approx$0 KL divergence at every budget, while all five eviction baselines degrade severely (5--28\% match, KL 7--14). The blue-bordered row highlights \kvdirect.}
\label{fig:compression}
\end{figure*}

The gap between methods is large. \kvdirect achieves 100\% token match and near-zero KL divergence ($<10^{-5}$) at \emph{every} budget on both models, matching the full (unbounded) KV cache exactly. All five eviction baselines, by contrast, degrade severely even at the most generous budget ($B{=}384$, 75\% retention): token match ranges from 6\% to 28\% and KL divergence from 7.5 to 14.1. The gap is not marginal; it spans orders of magnitude on KL and 70--95 percentage points on token match. At the most aggressive budget ($B{=}32$, 6\% retention), baselines produce essentially random output while \kvdirect remains lossless.

\paragraph{KV budget sweep.}
We also isolate the effect of cache window size without any eviction baseline. Holding only the last $B$ tokens in cache and evicting the rest without recomputation, we measure token match on two models across six window sizes ($B \in \{8, 16, 32, 64, 128, 256\}$) with 250-token generation. At $B{=}256$, Qwen2.5-0.5B recovers 88\% of tokens while SmolLM2-135M recovers only 34\%, reflecting model-specific sensitivity to context truncation. At the smallest window ($B{=}8$), both models produce near-random output (0--2\% match). With residual recomputation enabled, \kvdirect recovers 100\% match at every budget, because evicted tokens are recomputed from residual checkpoints rather than discarded.

\subsection{Effective Rank Analysis (RQ4)}
\label{sec:rank_analysis}

We compute $\bM^{(h)} = \bW_q^{(h)}{\bW_k^{(h)}}^\top$ for every attention head in three models and measure effective rank from the singular value spectrum. Figure~\ref{fig:rank_dotplot} shows the result as a dual-encoded dot matrix across all heads and layers: colour indicates the fraction of architectural rank used at 90\% spectral energy (blue $=$ low rank / highly compressible, red $=$ near full rank), while dot size encodes the same fraction as area. Dashed outlines mark Gemma's five global-attention layers.

\begin{figure*}
\centering
\includegraphics[width=\linewidth]{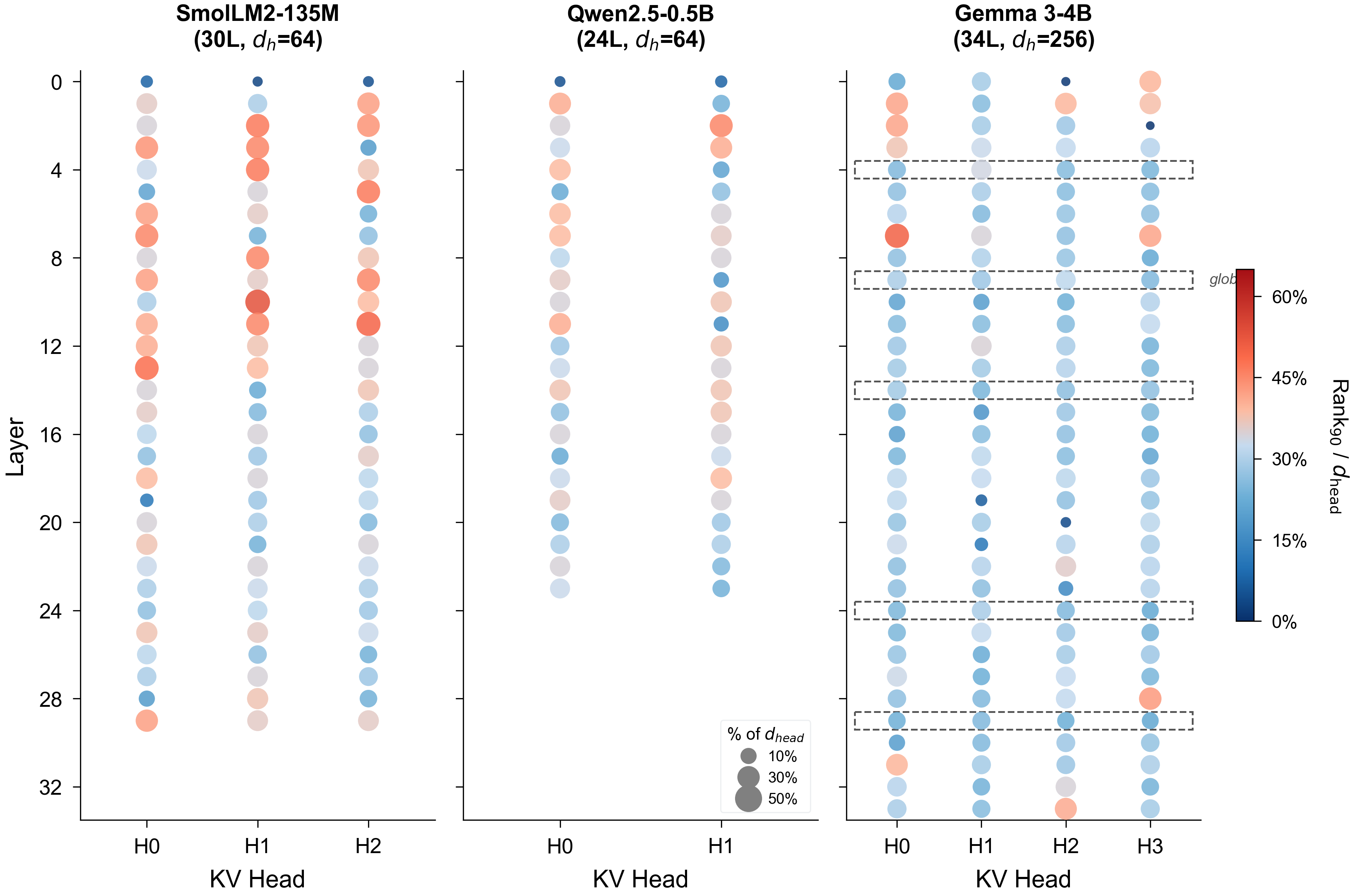}
\caption{Effective rank of $\bM^{(h)} = \bW_q^{(h)}{\bW_k^{(h)}}^\top$ at 90\% spectral energy across three models. Each dot is one KV head at one layer. \textbf{Colour}: rank as a fraction of $d_{\text{head}}$ (blue $=$ compressible, red $=$ near full rank). \textbf{Size}: same fraction (larger $=$ higher rank). Dashed outlines on Gemma mark global-attention layers. Layer~0 consistently shows near-rank-1 heads across all models, consistent with the BOS-focus phenomenon~\citep{xiao2024efficient}. Rank heterogeneity is visible both within and across architectures.}
\label{fig:rank_dotplot}
\end{figure*}

The three architectures differ in head dimension ($d_{\text{head}} = 64$ vs.\ 256) but share a common pattern visible in Figure~\ref{fig:rank_dotplot}: the mean effective rank at 90\% energy is 27--33\% of $d_{\text{head}}$ (mean ranks of 21.2, 20.0, and 70.2 for SmolLM2, Qwen2.5-0.5B, and Gemma respectively), yielding $3.0$--$3.6\times$ compression ratios. A long tail of near-rank-1 heads is present across all models; on Gemma, 3 of 136 heads (2\%) have effective rank $\leq 10$, including one rank-1 head at layer~0 consistent with the attention-sink phenomenon.

\paragraph{Low-rank approximation fails at generation.}
Despite the low effective rank, truncating the KV projections to rank $r < d_{\text{head}}$ destroys output quality. Figure~\ref{fig:rank_inference} plots token match and KL divergence against projection rank for SmolLM2-135M and Qwen2.5-0.5B. At full rank ($r{=}64$), token match is 100\% with $D_{\mathrm{KL}} = 5 \times 10^{-4}$. At $r{=}32$ (50\% of $d_{\text{head}}$), match drops to 15\% with $D_{\mathrm{KL}} = 10.9$. Below $r{=}15$, output is essentially random (5--10\% match, $D_{\mathrm{KL}} > 11$). The spectral energy captured by the top 32 components exceeds 95\%, yet discarding the remaining 5\% of energy produces catastrophic output degradation. This exposes a separation: the low-rank structure explains \emph{why} attention works (computation concentrates on a subspace), but it cannot be exploited for lossy compression without degrading generation. Lossless recomputation from the full residual stream, as in \kvdirect, is the only approach that preserves exact output fidelity.

\begin{figure*}
\centering
\includegraphics[width=\linewidth]{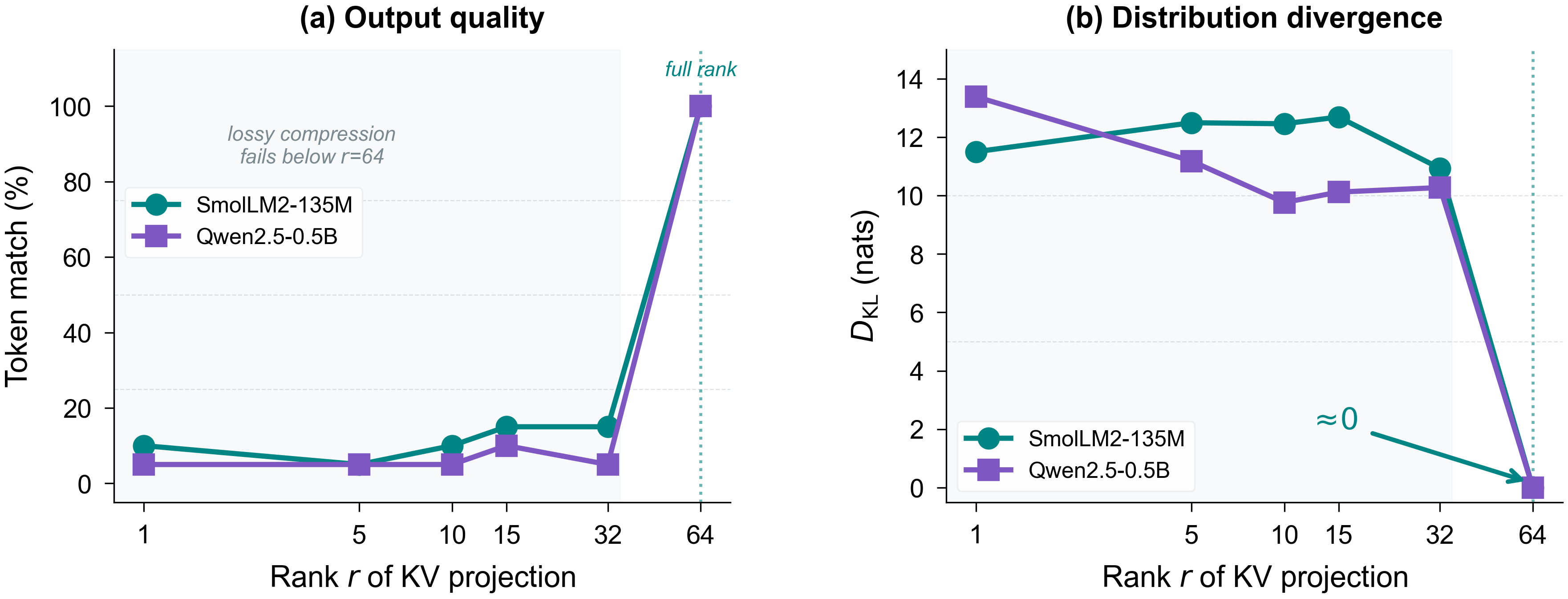}
\caption{Token match (\%) and KL divergence vs.\ KV projection rank $r$ on two models. At full rank ($r{=}64$), both models achieve 100\% match. Truncating to $r{=}32$ (50\% of $d_{\text{head}}$, capturing $>$95\% spectral energy) causes catastrophic degradation: 5--15\% match and $D_{\mathrm{KL}} > 10$. The shaded region marks ranks where lossy compression fails.}
\label{fig:rank_inference}
\end{figure*}

\paragraph{Geometric account of eviction robustness.}
This low-rank structure provides a geometric account of why token-importance eviction methods~\citep{zhang2024h2o, liu2024scissorhands} preserve generation quality at moderate budgets: attention computation concentrates along a small subspace of the residual stream, so tokens whose projections lie outside this subspace contribute minimally to the attention pattern. Rank-based compression is most effective when applied selectively to the near-rank-1 minority rather than uniformly across all heads.

\subsection{Recomputation Latency (RQ4)}
\label{sec:recompute_latency}

A primary assumption driving unbounded KV caches is that recomputing state is invariably slower than reading it from memory. To test this, we benchmarked the time required to reconstruct $N$ KV vectors from residual matrices versus copying identical cached tensors over the memory bus.

Figure~\ref{fig:latency} reveals a surprising crossover: memory bandwidth becomes the overriding bottleneck. For small batches of evicted tokens ($N{=}1$), recomputation holds a slight overhead ($1.1\times$ ratio). However, as $N$ scales, dense matrix multiplication from residuals fully outstrips memory fetches. At $N{=}500$, reconstructing the matrices from residuals operates in roughly $0.2\times$ to $0.3\times$ the time required to read pre-computed cache structures from memory. Checkpointing the residual is not merely a memory optimization; for moderately sized token windows, it accelerates data delivery to the attention compute units.

\begin{figure}
\centering
\includegraphics[width=\linewidth]{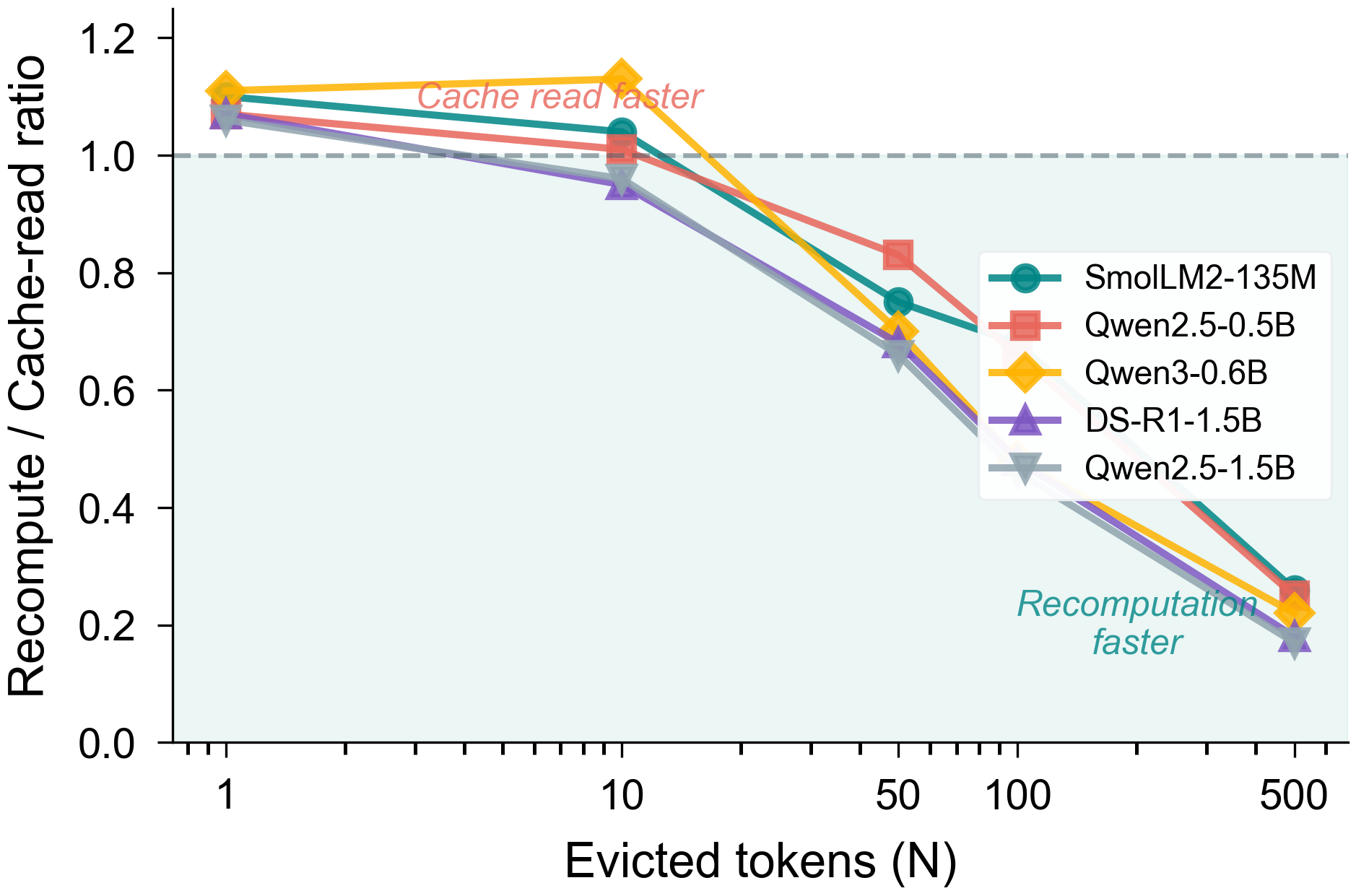}
\caption{Recompute-to-cache-read latency ratio across five model architectures. Each curve traces one model as the eviction batch size $N$ increases from 1 to 500 tokens. The teal-shaded region marks where recomputation is faster than cache reading. All models cross below parity by $N{=}50$.}
\label{fig:latency}
\end{figure}

All five models cross below parity by $N{=}50$. At $N{=}100$, ratios range from $0.46\times$ (Qwen2.5-1.5B) to $0.68\times$ (SmolLM2-135M). At $N{=}500$, the largest model reconstructs KV in $0.17\times$ the time required to read cached tensors. The crossover point scales with model size: larger models have proportionally more compute per byte of cache, so recomputation amortises faster.

\subsection{Discussion}
\label{sec:discussion}

Our results recast the KV cache as derived state. Serving systems~\citep{kwon2023efficient} currently treat KV entries as irreplaceable, engineering memory allocation, garbage collection, and CPU/disk swapping around them. Recognising KV as recoverable from the residual stream turns the cache into a true cache in the computer-science sense: a performance optimisation that can be evicted and regenerated without loss of correctness.

This shift has direct implications for edge deployment and long conversations. On a memory-constrained device, the KV cache is often the binding constraint on context length; residual checkpointing relaxes it by trading compute for memory. For long conversations, \kvdirect offers a third option beyond truncation and unbounded growth: retain all context in a fixed memory budget, recomputing KV for evicted tokens as needed. Combined with attention sparsity methods~\citep{zhang2024h2o, chen2024quest}, the recomputation cost drops in proportion to the sparsity. The approach also composes naturally with FlashAttention~\citep{dao2022flashattention, dao2023flashattention2}, which optimises within a single attention call while residual checkpointing optimises across the sequence dimension.

\section{Limitations and Future Work}
\label{sec:limitations}

Our experiments cover six models from 135M to 4B parameters. The theoretical argument applies to any pre-norm transformer with standard attention, but we have not verified exact-zero reconstruction on models with LayerNorm (instead of RMSNorm), mixture-of-experts routing, or parameters above 4B. On Gemma~3-4B-IT, cache-free recompute degrades HellaSwag accuracy from 49.2\% to 25.0\% and perplexity by orders of magnitude, confirming that sliding-window architectures require position-aware cache management that simple recompute does not provide (Section~\ref{sec:downstream_eval}). The multi-turn benchmark uses 20 turns; real-world deployment at 100K+ token contexts would face additional recomputation latency.

The bounded inference prototype uses full recomputation (budget $B{=}0$), the extreme case. A practical system would set $B > 0$, caching recent tokens and recomputing only evicted entries. The optimal budget depends on the hardware's compute-to-memory bandwidth ratio, a system-level optimisation we leave to future work.

Several directions warrant investigation. First, integrating residual checkpointing into production serving frameworks (e.g., vLLM) to measure end-to-end throughput at scale. Second, combining \kvdirect with weight quantisation and per-head mixed-precision caching guided by the rank heterogeneity observed in Section~\ref{sec:rank_analysis}. Third, extending the Markov property analysis to architectures with cross-layer KV sharing~\citep{zhang2024batchedgeneralization, wu2024layer} and latent-space attention~\citep{deepseekv2}, where the residual-to-KV mapping takes different algebraic forms.

\section{Conclusion}
\label{sec:conclusion}

We have shown, through theory and experiment, that the KV cache in transformer inference is a computational shortcut, not an information store. Keys and values at every layer are deterministic projections of the residual stream. Removing the cache and recomputing from scratch produces identical tokens. Replacing the residual stream wholesale at any layer produces the donor's output distribution with zero KL divergence, confirming the Markov property.

The bilinear attention form reveals strong rank heterogeneity across heads and models. The mean effective rank at 90\% spectral energy is 27--33\% of $d_{\text{head}}$ on all three architectures tested, with a small subset ($<$2\%) of near-rank-1 attention-sink heads. This structure motivates per-head mixed-precision caching and provides a geometric account of why token-eviction heuristics preserve generation quality.

These results reframe KV cache management. Instead of designing eviction policies that try to preserve the ``most important'' cached entries, we can treat the residual stream as ground truth and recompute KV entries as needed. Memory becomes a tunable knob: more cache means faster inference, less cache means lower memory, but correctness is guaranteed regardless. For memory-constrained settings (edge devices, long conversations, high-concurrency serving), this guarantee changes what is possible.

\section*{Acknowledgments}
The authors extend their appreciation to the National Science Foundation of China under grants (No.:62471411).

\bibliographystyle{model1-num-names}

\begin{thebibliography}{49}
\expandafter\ifx\csname natexlab\endcsname\relax\def\natexlab#1{#1}\fi
\providecommand{\url}[1]{\texttt{#1}}
\providecommand{\href}[2]{#2}
\providecommand{\path}[1]{#1}
\providecommand{\DOIprefix}{doi:}
\providecommand{\ArXivprefix}{arXiv:}
\providecommand{\URLprefix}{URL: }
\providecommand{\Pubmedprefix}{pmid:}
\providecommand{\doi}[1]{\href{http://dx.doi.org/#1}{\path{#1}}}
\providecommand{\Pubmed}[1]{\href{pmid:#1}{\path{#1}}}
\providecommand{\bibinfo}[2]{#2}
\ifx\xfnm\relax \def\xfnm[#1]{\unskip,\space#1}\fi
\bibitem[{Zhang et~al.(2023)Zhang, Sheng, Zhou, Chen, Zheng, Cai, Song, Tian,
  R{\'e}, Barrett et~al.}]{zhang2024h2o}
\bibinfo{author}{Z.~Zhang}, \bibinfo{author}{Y.~Sheng},
  \bibinfo{author}{T.~Zhou}, \bibinfo{author}{T.~Chen},
  \bibinfo{author}{L.~Zheng}, \bibinfo{author}{R.~Cai},
  \bibinfo{author}{Z.~Song}, \bibinfo{author}{Y.~Tian},
  \bibinfo{author}{C.~R{\'e}}, \bibinfo{author}{C.~Barrett}, et~al.,
\newblock \bibinfo{title}{{H2O}: Heavy-hitter oracle for efficient generative
  inference of large language models},
\newblock in: \bibinfo{booktitle}{Advances in Neural Information Processing
  Systems}, volume~\bibinfo{volume}{36}, \bibinfo{year}{2023}, pp.
  \bibinfo{pages}{34661--34710}.
\bibitem[{Liu et~al.(2023)Liu, Desai, Liao, Wang, Xie, Xu, Kyrillidis, and
  Shrivastava}]{liu2024scissorhands}
\bibinfo{author}{Z.~Liu}, \bibinfo{author}{A.~Desai},
  \bibinfo{author}{F.~Liao}, \bibinfo{author}{W.~Wang},
  \bibinfo{author}{V.~Xie}, \bibinfo{author}{Z.~Xu},
  \bibinfo{author}{A.~Kyrillidis}, \bibinfo{author}{A.~Shrivastava},
\newblock \bibinfo{title}{Scissorhands: Exploiting the persistence of
  importance hypothesis for {LLM} {KV} cache compression at test time},
\newblock \bibinfo{journal}{Advances in Neural Information Processing Systems}
  \bibinfo{volume}{36} (\bibinfo{year}{2023}) \bibinfo{pages}{52342--52364}.
\bibitem[{Devoto et~al.(2024)Devoto, Zhao, Scardapane, and
  Minervini}]{devoto2024simple}
\bibinfo{author}{A.~Devoto}, \bibinfo{author}{Y.~Zhao},
  \bibinfo{author}{S.~Scardapane}, \bibinfo{author}{P.~Minervini},
\newblock \bibinfo{title}{A simple and effective {L2} norm-based strategy for
  {KV} cache compression},
\newblock \bibinfo{journal}{arXiv preprint arXiv:2406.11430}
  (\bibinfo{year}{2024}).
\bibitem[{Gong et~al.(2025)Gong, Ding, Wang, Lv, Zheng, Du, Yong, Gu, Qin, Guo,
  Lin, Magno, and Liu}]{gong2025lowbit}
\bibinfo{author}{R.~Gong}, \bibinfo{author}{Y.~Ding},
  \bibinfo{author}{Z.~Wang}, \bibinfo{author}{C.~Lv},
  \bibinfo{author}{X.~Zheng}, \bibinfo{author}{J.~Du},
  \bibinfo{author}{Y.~Yong}, \bibinfo{author}{S.~Gu}, \bibinfo{author}{H.~Qin},
  \bibinfo{author}{J.~Guo}, \bibinfo{author}{D.~Lin},
  \bibinfo{author}{M.~Magno}, \bibinfo{author}{X.~Liu},
\newblock \bibinfo{title}{A survey of low-bit large language models: Basics,
  systems, and algorithms},
\newblock \bibinfo{journal}{Neural Networks} \bibinfo{volume}{192}
  (\bibinfo{year}{2025}) \bibinfo{pages}{107856}.
\bibitem[{Ainslie et~al.(2023)Ainslie, Lee-Thorp, de~Jong, Zemlyanskiy,
  Lebr{\'o}n, and Sanghai}]{ainslie2023gqa}
\bibinfo{author}{J.~Ainslie}, \bibinfo{author}{J.~Lee-Thorp},
  \bibinfo{author}{M.~de~Jong}, \bibinfo{author}{Y.~Zemlyanskiy},
  \bibinfo{author}{F.~Lebr{\'o}n}, \bibinfo{author}{S.~Sanghai},
\newblock \bibinfo{title}{{GQA}: Training generalized multi-query transformers
  from multi-head checkpoints},
\newblock in: \bibinfo{booktitle}{Proceedings of the 2023 Conference on
  Empirical Methods in Natural Language Processing}, \bibinfo{year}{2023}.
\bibitem[{Kwon et~al.(2023)Kwon, Li, Zhuang, Sheng, Zheng, Yu, Gonzalez, Zhang,
  and Stoica}]{kwon2023efficient}
\bibinfo{author}{W.~Kwon}, \bibinfo{author}{Z.~Li},
  \bibinfo{author}{S.~Zhuang}, \bibinfo{author}{Y.~Sheng},
  \bibinfo{author}{L.~Zheng}, \bibinfo{author}{C.~H. Yu},
  \bibinfo{author}{J.~Gonzalez}, \bibinfo{author}{H.~Zhang},
  \bibinfo{author}{I.~Stoica},
\newblock \bibinfo{title}{Efficient memory management for large language model
  serving with {PagedAttention}},
\newblock in: \bibinfo{booktitle}{Proceedings of the 29th Symposium on
  Operating Systems Principles}, \bibinfo{year}{2023}.
\bibitem[{Zellers et~al.(2019)Zellers, Holtzman, Bisk, Farhadi, and
  Choi}]{zellers2019hellaswag}
\bibinfo{author}{R.~Zellers}, \bibinfo{author}{A.~Holtzman},
  \bibinfo{author}{Y.~Bisk}, \bibinfo{author}{A.~Farhadi},
  \bibinfo{author}{Y.~Choi},
\newblock \bibinfo{title}{{HellaSwag}: Can a machine really finish your
  sentence?},
\newblock in: \bibinfo{booktitle}{Proceedings of the 57th Annual Meeting of the
  Association for Computational Linguistics}, \bibinfo{year}{2019}, pp.
  \bibinfo{pages}{4791--4800}.
\bibitem[{Xiao et~al.(2024)Xiao, Tian, Chen, Han, and
  Lewis}]{xiao2024efficient}
\bibinfo{author}{G.~Xiao}, \bibinfo{author}{Y.~Tian},
  \bibinfo{author}{B.~Chen}, \bibinfo{author}{S.~Han},
  \bibinfo{author}{M.~Lewis},
\newblock \bibinfo{title}{Efficient streaming language models with attention
  sinks},
\newblock \bibinfo{journal}{arXiv preprint arXiv:2309.17453}
  (\bibinfo{year}{2024}).
\bibitem[{Li et~al.(2024)Li, Huang, Yang, Venkitesh, Locatelli, Ye, Cai, Lewis,
  and Chen}]{li2024snapkv}
\bibinfo{author}{Y.~Li}, \bibinfo{author}{Y.~Huang}, \bibinfo{author}{B.~Yang},
  \bibinfo{author}{B.~Venkitesh}, \bibinfo{author}{A.~Locatelli},
  \bibinfo{author}{H.~Ye}, \bibinfo{author}{T.~Cai},
  \bibinfo{author}{P.~Lewis}, \bibinfo{author}{D.~Chen},
\newblock \bibinfo{title}{{SnapKV}: {LLM} knows what you are looking for before
  generation},
\newblock in: \bibinfo{booktitle}{Advances in Neural Information Processing
  Systems}, volume~\bibinfo{volume}{37}, \bibinfo{year}{2024}.
\bibitem[{Oren et~al.(2024)Oren, Hassid, Adi, and
  Schwartz}]{oren2024transformers}
\bibinfo{author}{M.~Oren}, \bibinfo{author}{M.~Hassid},
  \bibinfo{author}{Y.~Adi}, \bibinfo{author}{R.~Schwartz},
\newblock \bibinfo{title}{Transformers are multi-state {RNNs}},
\newblock in: \bibinfo{booktitle}{Proceedings of the 2024 Conference on
  Empirical Methods in Natural Language Processing}, \bibinfo{year}{2024}.
\bibitem[{Zhang et~al.(2024)Zhang, Yang et~al.}]{zhang2024pyramidkv}
\bibinfo{author}{Z.~Zhang}, \bibinfo{author}{Z.~Yang}, et~al.,
\newblock \bibinfo{title}{{PyramidKV}: Dynamic {KV} cache compression based on
  pyramidal information funneling},
\newblock \bibinfo{journal}{arXiv preprint arXiv:2406.02069}
  (\bibinfo{year}{2024}).
\bibitem[{Yang et~al.(2024)Yang, Han, Gao, Hu, Zhang, and
  Zhao}]{yang2024pyramidinfer}
\bibinfo{author}{D.~Yang}, \bibinfo{author}{X.~Han}, \bibinfo{author}{Y.~Gao},
  \bibinfo{author}{Y.~Hu}, \bibinfo{author}{S.~Zhang},
  \bibinfo{author}{H.~Zhao},
\newblock \bibinfo{title}{{PyramidInfer}: Pyramid {KV} cache compression for
  high-throughput {LLM} inference},
\newblock in: \bibinfo{booktitle}{Findings of the Association for Computational
  Linguistics: ACL 2024}, \bibinfo{year}{2024}.
\bibitem[{Tang et~al.(2024)Tang, Zhao, Zhu, Xiao, Kasikci, and
  Han}]{chen2024quest}
\bibinfo{author}{J.~Tang}, \bibinfo{author}{Y.~Zhao}, \bibinfo{author}{K.~Zhu},
  \bibinfo{author}{G.~Xiao}, \bibinfo{author}{B.~Kasikci},
  \bibinfo{author}{S.~Han},
\newblock \bibinfo{title}{Quest: Query-aware sparsity for efficient
  long-context {LLM} inference},
\newblock \bibinfo{journal}{arXiv preprint arXiv:2406.10774}
  (\bibinfo{year}{2024}).
\bibitem[{Ge et~al.(2024)Ge, Zhang, Liu, Zhang, Han, and Gao}]{ge2024model}
\bibinfo{author}{S.~Ge}, \bibinfo{author}{Y.~Zhang}, \bibinfo{author}{L.~Liu},
  \bibinfo{author}{M.~Zhang}, \bibinfo{author}{J.~Han},
  \bibinfo{author}{J.~Gao},
\newblock \bibinfo{title}{Model tells you what to discard: Adaptive {KV} cache
  compression for {LLMs}},
\newblock in: \bibinfo{booktitle}{International Conference on Learning
  Representations}, \bibinfo{year}{2024}.
\bibitem[{Tang et~al.(2025)Tang, Lin, Lin, Han, Hong, Yao, and
  Wang}]{tang2025razorattention}
\bibinfo{author}{H.~Tang}, \bibinfo{author}{Y.~Lin}, \bibinfo{author}{J.~Lin},
  \bibinfo{author}{Q.~Han}, \bibinfo{author}{S.~Hong},
  \bibinfo{author}{Y.~Yao}, \bibinfo{author}{G.~Wang},
\newblock \bibinfo{title}{{RazorAttention}: Efficient {KV} cache compression
  through retrieval heads},
\newblock in: \bibinfo{booktitle}{International Conference on Learning
  Representations}, \bibinfo{year}{2025}.
\bibitem[{Yuan et~al.(2024)Yuan, Lv, Zhou et~al.}]{yuan2024adakv}
\bibinfo{author}{F.~Yuan}, \bibinfo{author}{J.~Lv}, \bibinfo{author}{J.~Zhou},
  et~al.,
\newblock \bibinfo{title}{Ada-{KV}: Optimizing {KV} cache eviction by adaptive
  budget allocation for efficient {LLM} inference},
\newblock \bibinfo{journal}{arXiv preprint arXiv:2407.11550}
  (\bibinfo{year}{2024}).
\bibitem[{Liu et~al.(2024)Liu, Yuan, Jin, Zhong, Xu, Braverman, Chen, and
  Hu}]{liu2024kivi}
\bibinfo{author}{Z.~Liu}, \bibinfo{author}{J.~Yuan}, \bibinfo{author}{H.~Jin},
  \bibinfo{author}{S.~Zhong}, \bibinfo{author}{Z.~Xu},
  \bibinfo{author}{V.~Braverman}, \bibinfo{author}{B.~Chen},
  \bibinfo{author}{X.~Hu},
\newblock \bibinfo{title}{{KIVI}: A tuning-free asymmetric 2bit quantization
  for {KV} cache},
\newblock in: \bibinfo{booktitle}{International Conference on Machine
  Learning}, \bibinfo{year}{2024}.
\bibitem[{Hooper et~al.(2024)Hooper, Kim, Mohber, Wattanawong, Mahoney, Shao,
  Keutzer, and Gholami}]{hooper2024kvquant}
\bibinfo{author}{C.~Hooper}, \bibinfo{author}{S.~Kim},
  \bibinfo{author}{H.~Mohber}, \bibinfo{author}{T.~Wattanawong},
  \bibinfo{author}{M.~W. Mahoney}, \bibinfo{author}{Y.~S. Shao},
  \bibinfo{author}{K.~Keutzer}, \bibinfo{author}{A.~Gholami},
\newblock \bibinfo{title}{{KVQuant}: Towards 10 million context length {LLM}
  inference with {KV} cache quantization},
\newblock in: \bibinfo{booktitle}{Advances in Neural Information Processing
  Systems}, volume~\bibinfo{volume}{37}, \bibinfo{year}{2024}.
\bibitem[{Kang et~al.(2024)Kang, Zhang, Kundu, Jeong, Liu, Krishna, and
  Zhao}]{kang2024gear}
\bibinfo{author}{H.~Kang}, \bibinfo{author}{Q.~Zhang},
  \bibinfo{author}{S.~Kundu}, \bibinfo{author}{G.~Jeong},
  \bibinfo{author}{Z.~Liu}, \bibinfo{author}{T.~Krishna},
  \bibinfo{author}{T.~Zhao},
\newblock \bibinfo{title}{{GEAR}: An efficient {KV} cache compression recipe
  for near-lossless generative inference of {LLM}},
\newblock \bibinfo{journal}{arXiv preprint arXiv:2403.05527}
  (\bibinfo{year}{2024}).
\bibitem[{Zhang et~al.(2024)Zhang, Yi, Xu, and
  Shrivastava}]{zhang2024kvcache1bit}
\bibinfo{author}{T.~Zhang}, \bibinfo{author}{J.~Yi}, \bibinfo{author}{Z.~Xu},
  \bibinfo{author}{A.~Shrivastava},
\newblock \bibinfo{title}{{KV} cache is 1 bit per channel: Efficient large
  language model inference with coupled quantization},
\newblock in: \bibinfo{booktitle}{Advances in Neural Information Processing
  Systems}, volume~\bibinfo{volume}{37}, \bibinfo{year}{2024}.
\bibitem[{Wang et~al.(2025)Wang, Han, Guo, He, Gao, and Lu}]{wang2025mpoq}
\bibinfo{author}{J.-Q. Wang}, \bibinfo{author}{X.-Q. Han},
  \bibinfo{author}{P.-J. Guo}, \bibinfo{author}{R.-Q. He},
  \bibinfo{author}{Z.-F. Gao}, \bibinfo{author}{Z.-Y. Lu},
\newblock \bibinfo{title}{Enabling efficient low-bit quantization based on
  matrix product operators for {KV} cache compression},
\newblock \bibinfo{journal}{Neural Networks} \bibinfo{volume}{197}
  (\bibinfo{year}{2025}) \bibinfo{pages}{108467}.
\bibitem[{Chang et~al.(2025)Chang, Lin, Lin, Chen, Hu, Wang, Huang, Ceze,
  Abdelfattah, and Wu}]{chang2025palu}
\bibinfo{author}{C.-C. Chang}, \bibinfo{author}{W.-C. Lin},
  \bibinfo{author}{C.-Y. Lin}, \bibinfo{author}{C.-Y. Chen},
  \bibinfo{author}{Y.-F. Hu}, \bibinfo{author}{P.-H. Wang},
  \bibinfo{author}{N.-C. Huang}, \bibinfo{author}{L.~Ceze},
  \bibinfo{author}{M.~S. Abdelfattah}, \bibinfo{author}{K.-C. Wu},
\newblock \bibinfo{title}{Palu: Compressing {KV}-cache with low-rank
  projection},
\newblock in: \bibinfo{booktitle}{International Conference on Learning
  Representations}, \bibinfo{year}{2025}.
\bibitem[{Singhania et~al.(2024)Singhania, Singh, He, Feizi, and
  Bhatele}]{singhania2024loki}
\bibinfo{author}{P.~Singhania}, \bibinfo{author}{S.~Singh},
  \bibinfo{author}{S.~He}, \bibinfo{author}{S.~Feizi},
  \bibinfo{author}{A.~Bhatele},
\newblock \bibinfo{title}{Loki: Low-rank keys for efficient sparse attention},
\newblock in: \bibinfo{booktitle}{Advances in Neural Information Processing
  Systems}, volume~\bibinfo{volume}{37}, \bibinfo{year}{2024}.
\bibitem[{Saxena et~al.(2024)Saxena, Saha, Choudhary, and
  Roy}]{saxena2024eigen}
\bibinfo{author}{U.~Saxena}, \bibinfo{author}{G.~Saha},
  \bibinfo{author}{S.~Choudhary}, \bibinfo{author}{K.~Roy},
\newblock \bibinfo{title}{Eigen attention: Attention in low-rank space for {KV}
  cache compression},
\newblock in: \bibinfo{booktitle}{Findings of the Association for Computational
  Linguistics: EMNLP 2024}, \bibinfo{year}{2024}.
\bibitem[{Shazeer(2019)}]{shazeer2019fast}
\bibinfo{author}{N.~Shazeer},
\newblock \bibinfo{title}{Fast transformer decoding: One write-head is all you
  need},
\newblock \bibinfo{journal}{arXiv preprint arXiv:1911.02150}
  (\bibinfo{year}{2019}).
\bibitem[{Brandon et~al.(2024)Brandon, Mishra, Nrusimha, Panda, and
  Kelly}]{zhang2024batchedgeneralization}
\bibinfo{author}{W.~Brandon}, \bibinfo{author}{M.~Mishra},
  \bibinfo{author}{A.~Nrusimha}, \bibinfo{author}{R.~Panda},
  \bibinfo{author}{J.~R. Kelly},
\newblock \bibinfo{title}{Reducing transformer key-value cache size with
  cross-layer attention},
\newblock \bibinfo{journal}{arXiv preprint arXiv:2405.12981}
  (\bibinfo{year}{2024}).
\bibitem[{Wu and Tu(2024)}]{wu2024layer}
\bibinfo{author}{H.~Wu}, \bibinfo{author}{K.~Tu},
\newblock \bibinfo{title}{Layer-condensed {KV} cache for efficient inference of
  large language models},
\newblock in: \bibinfo{booktitle}{Proceedings of the 62nd Annual Meeting of the
  Association for Computational Linguistics}, \bibinfo{year}{2024}.
\bibitem[{Sun et~al.(2024)Sun, Dong, Zhu, Huang, Wang, Ma, Zhang, Wang, and
  Wei}]{sun2024yoco}
\bibinfo{author}{Y.~Sun}, \bibinfo{author}{L.~Dong}, \bibinfo{author}{Y.~Zhu},
  \bibinfo{author}{S.~Huang}, \bibinfo{author}{W.~Wang},
  \bibinfo{author}{S.~Ma}, \bibinfo{author}{Q.~Zhang},
  \bibinfo{author}{J.~Wang}, \bibinfo{author}{F.~Wei},
\newblock \bibinfo{title}{You only cache once: Decoder-decoder architectures
  for language models},
\newblock in: \bibinfo{booktitle}{Advances in Neural Information Processing
  Systems}, volume~\bibinfo{volume}{37}, \bibinfo{year}{2024}.
\bibitem[{{DeepSeek-AI}(2024)}]{deepseekv2}
\bibinfo{author}{{DeepSeek-AI}},
\newblock \bibinfo{title}{{DeepSeek-V2}: A strong, economical, and efficient
  mixture-of-experts language model},
\newblock \bibinfo{journal}{arXiv preprint arXiv:2405.04434}
  (\bibinfo{year}{2024}).
\bibitem[{Liu et~al.(2024)Liu, Liu et~al.}]{liu2024minicache}
\bibinfo{author}{A.~Liu}, \bibinfo{author}{J.~Liu}, et~al.,
\newblock \bibinfo{title}{{MiniCache}: {KV} cache compression in depth
  dimension for large language models},
\newblock in: \bibinfo{booktitle}{Advances in Neural Information Processing
  Systems}, volume~\bibinfo{volume}{37}, \bibinfo{year}{2024}.
\bibitem[{Jiang et~al.(2025)Jiang, Gao, Zarch, and Annavaram}]{kvpr2025}
\bibinfo{author}{C.~Jiang}, \bibinfo{author}{L.~Gao}, \bibinfo{author}{H.~E.
  Zarch}, \bibinfo{author}{M.~Annavaram},
\newblock \bibinfo{title}{{KVPR}: Efficient {LLM} inference with {I/O}-aware
  {KV} cache partial recomputation},
\newblock \bibinfo{journal}{arXiv preprint arXiv:2411.17089}
  (\bibinfo{year}{2025}).
\bibitem[{Lee et~al.(2025)Lee, Kim, Hwang, Heo, Noh, and Huh}]{hybridserve2025}
\bibinfo{author}{S.~Lee}, \bibinfo{author}{H.~Kim}, \bibinfo{author}{S.~Hwang},
  \bibinfo{author}{G.~Heo}, \bibinfo{author}{M.~Noh}, \bibinfo{author}{J.~Huh},
\newblock \bibinfo{title}{Efficient {LLM} inference with activation
  checkpointing and hybrid caching},
\newblock \bibinfo{journal}{arXiv preprint arXiv:2501.01792}
  (\bibinfo{year}{2025}).
\bibitem[{Dao et~al.(2022)Dao, Fu, Ermon, Rudra, and
  R{\'e}}]{dao2022flashattention}
\bibinfo{author}{T.~Dao}, \bibinfo{author}{D.~Fu}, \bibinfo{author}{S.~Ermon},
  \bibinfo{author}{A.~Rudra}, \bibinfo{author}{C.~R{\'e}},
\newblock \bibinfo{title}{{FlashAttention}: Fast and memory-efficient exact
  attention with {IO}-awareness} \bibinfo{volume}{35} (\bibinfo{year}{2022}).
\bibitem[{Dao(2023)}]{dao2023flashattention2}
\bibinfo{author}{T.~Dao},
\newblock \bibinfo{title}{{FlashAttention-2}: Faster attention with better
  parallelism and work partitioning},
\newblock \bibinfo{journal}{arXiv preprint arXiv:2307.08691}
  (\bibinfo{year}{2023}).
\bibitem[{Chen et~al.(2016)Chen, Xu, Zhang, and Guestrin}]{chen2016training}
\bibinfo{author}{T.~Chen}, \bibinfo{author}{B.~Xu}, \bibinfo{author}{C.~Zhang},
  \bibinfo{author}{C.~Guestrin},
\newblock \bibinfo{title}{Training deep nets with sublinear memory cost},
\newblock \bibinfo{journal}{arXiv preprint arXiv:1604.06174}
  (\bibinfo{year}{2016}).
\bibitem[{Elhage et~al.(2021)Elhage, Nanda, Olsson, Henighan, Joseph, Mann,
  Askell, Bai, Chen, Conerly et~al.}]{elhage2021mathematical}
\bibinfo{author}{N.~Elhage}, \bibinfo{author}{N.~Nanda},
  \bibinfo{author}{C.~Olsson}, \bibinfo{author}{T.~Henighan},
  \bibinfo{author}{N.~Joseph}, \bibinfo{author}{B.~Mann},
  \bibinfo{author}{A.~Askell}, \bibinfo{author}{Y.~Bai},
  \bibinfo{author}{A.~Chen}, \bibinfo{author}{T.~Conerly}, et~al.,
\newblock \bibinfo{title}{A mathematical framework for transformer circuits},
\newblock \bibinfo{journal}{Transformer Circuits Thread}
  (\bibinfo{year}{2021}).
\bibitem[{Olsson et~al.(2022)Olsson, Elhage, Nanda, Joseph, DasSarma, Henighan,
  Mann, Askell, Bai, Chen et~al.}]{olsson2022context}
\bibinfo{author}{C.~Olsson}, \bibinfo{author}{N.~Elhage},
  \bibinfo{author}{N.~Nanda}, \bibinfo{author}{N.~Joseph},
  \bibinfo{author}{N.~DasSarma}, \bibinfo{author}{T.~Henighan},
  \bibinfo{author}{B.~Mann}, \bibinfo{author}{A.~Askell},
  \bibinfo{author}{Y.~Bai}, \bibinfo{author}{A.~Chen}, et~al.,
\newblock \bibinfo{title}{In-context learning and induction heads},
\newblock \bibinfo{journal}{Transformer Circuits Thread}
  (\bibinfo{year}{2022}).
\bibitem[{Shai et~al.(2024)Shai, Marzen, Teixeira, Oldenziel, and
  Riechers}]{shai2024transformers}
\bibinfo{author}{A.~Shai}, \bibinfo{author}{S.~Marzen},
  \bibinfo{author}{L.~Teixeira}, \bibinfo{author}{A.~G. Oldenziel},
  \bibinfo{author}{P.~M. Riechers},
\newblock \bibinfo{title}{Transformers represent belief state geometry in their
  residual stream},
\newblock in: \bibinfo{booktitle}{Advances in Neural Information Processing
  Systems}, volume~\bibinfo{volume}{37}, \bibinfo{year}{2024}.
\bibitem[{He et~al.(2024)He, Sun, Shen, and Li}]{he2024what}
\bibinfo{author}{S.~He}, \bibinfo{author}{G.~Sun}, \bibinfo{author}{Z.~Shen},
  \bibinfo{author}{A.~Li},
\newblock \bibinfo{title}{What matters in transformers? not all attention is
  needed},
\newblock \bibinfo{journal}{arXiv preprint arXiv:2406.15786}
  (\bibinfo{year}{2024}).
\bibitem[{Geiger et~al.(2021)Geiger, Lu, Icard, and Potts}]{geiger2021causal}
\bibinfo{author}{A.~Geiger}, \bibinfo{author}{H.~Lu},
  \bibinfo{author}{T.~Icard}, \bibinfo{author}{C.~Potts},
\newblock \bibinfo{title}{Causal abstractions of neural networks},
\newblock \bibinfo{journal}{Advances in Neural Information Processing Systems}
  \bibinfo{volume}{34} (\bibinfo{year}{2021}) \bibinfo{pages}{9574--9586}.
\bibitem[{Conmy et~al.(2023)Conmy, Mavor-Parker, Lynch, Heimersheim, and
  Garriga-Alonso}]{conmy2023towards}
\bibinfo{author}{A.~Conmy}, \bibinfo{author}{A.~N. Mavor-Parker},
  \bibinfo{author}{A.~Lynch}, \bibinfo{author}{S.~Heimersheim},
  \bibinfo{author}{A.~Garriga-Alonso},
\newblock \bibinfo{title}{Towards automated circuit discovery for mechanistic
  interpretability},
\newblock \bibinfo{journal}{Advances in Neural Information Processing Systems}
  \bibinfo{volume}{36} (\bibinfo{year}{2023}) \bibinfo{pages}{16318--16352}.
\bibitem[{Vaswani et~al.(2017)Vaswani, Shazeer, Parmar, Uszkoreit, Jones,
  Gomez, Kaiser, and Polosukhin}]{vaswani2017attention}
\bibinfo{author}{A.~Vaswani}, \bibinfo{author}{N.~Shazeer},
  \bibinfo{author}{N.~Parmar}, \bibinfo{author}{J.~Uszkoreit},
  \bibinfo{author}{L.~Jones}, \bibinfo{author}{A.~N. Gomez},
  \bibinfo{author}{{\L}.~Kaiser}, \bibinfo{author}{I.~Polosukhin},
\newblock \bibinfo{title}{Attention is all you need},
\newblock in: \bibinfo{booktitle}{Advances in Neural Information Processing
  Systems}, volume~\bibinfo{volume}{30}, \bibinfo{year}{2017}.
\bibitem[{He et~al.(2016)He, Zhang, Ren, and Sun}]{he2016deep}
\bibinfo{author}{K.~He}, \bibinfo{author}{X.~Zhang}, \bibinfo{author}{S.~Ren},
  \bibinfo{author}{J.~Sun},
\newblock \bibinfo{title}{Deep residual learning for image recognition}
  (\bibinfo{year}{2016}) \bibinfo{pages}{770--778}.
\bibitem[{Su et~al.(2024)Su, Ahmed, Lu, Pan, Bo, and Liu}]{su2024roformer}
\bibinfo{author}{J.~Su}, \bibinfo{author}{M.~Ahmed}, \bibinfo{author}{Y.~Lu},
  \bibinfo{author}{S.~Pan}, \bibinfo{author}{W.~Bo}, \bibinfo{author}{Y.~Liu},
\newblock \bibinfo{title}{{RoFormer}: Enhanced transformer with rotary position
  embedding},
\newblock \bibinfo{journal}{Neurocomputing} \bibinfo{volume}{568}
  (\bibinfo{year}{2024}) \bibinfo{pages}{127063}.
\bibitem[{Allal et~al.(2025)Allal, Lozhkov, Penedo, Wolf, and von
  Werra}]{allal2025smollm2}
\bibinfo{author}{L.~B. Allal}, \bibinfo{author}{A.~Lozhkov},
  \bibinfo{author}{G.~Penedo}, \bibinfo{author}{T.~Wolf},
  \bibinfo{author}{L.~von Werra},
\newblock \bibinfo{title}{{SmolLM2}: When smol goes big -- data-centric
  training of a small language model},
\newblock \bibinfo{journal}{arXiv preprint arXiv:2502.02737}
  (\bibinfo{year}{2025}).
\bibitem[{Team(2025)}]{qwen2025qwen25}
\bibinfo{author}{Q.~Team},
\newblock \bibinfo{title}{Qwen2.5 technical report},
\newblock \bibinfo{journal}{arXiv preprint arXiv:2412.15115}
  (\bibinfo{year}{2025}).
\bibitem[{Guo et~al.(2025)Guo, Yang, Zhang, Song, Zhang, Xu, Zhu, Ma, Wang, Bi
  et~al.}]{guo2025deepseek}
\bibinfo{author}{D.~Guo}, \bibinfo{author}{D.~Yang},
  \bibinfo{author}{H.~Zhang}, \bibinfo{author}{J.~Song},
  \bibinfo{author}{R.~Zhang}, \bibinfo{author}{R.~Xu},
  \bibinfo{author}{Q.~Zhu}, \bibinfo{author}{S.~Ma}, \bibinfo{author}{P.~Wang},
  \bibinfo{author}{X.~Bi}, et~al.,
\newblock \bibinfo{title}{Deepseek-r1: Incentivizing reasoning capability in
  llms via reinforcement learning},
\newblock \bibinfo{journal}{arXiv preprint arXiv:2501.12948}
  (\bibinfo{year}{2025}).
\bibitem[{{Gemma Team}(2025)}]{team2024gemma}
\bibinfo{author}{{Gemma Team}},
\newblock \bibinfo{title}{Gemma 3 technical report},
\newblock \bibinfo{journal}{arXiv preprint arXiv:2503.19786}
  (\bibinfo{year}{2025}).
\bibitem[{Hay(2025)}]{hay2025kvdirect}
\bibinfo{author}{C.~Hay}, \bibinfo{title}{We don't need {KV} cache anymore?
  {KV-Direct}: Bounded-memory inference via residual checkpointing},
  \bibinfo{howpublished}{\url{https://github.com/chrishayuk/chuk-lazarus}},
  \bibinfo{year}{2025}.

\end{thebibliography}

\appendix

\section{Experimental Details}
\label{app:details}

\paragraph{Hardware.}
All experiments were run on an Apple M3~Max with 64~GB unified memory. We use the MLX framework for model loading and inference, with bfloat16 precision throughout.

\paragraph{Models.}
\begin{itemize}[leftmargin=2em,topsep=2pt,itemsep=1pt]
  \item \textbf{SmolLM2-135M-Instruct}~\citep{allal2025smollm2}: LLaMA-family, 30 layers, $d_{\text{hidden}}=576$, 9 query heads, 3 KV heads, $d_{\text{head}}=64$.
  \item \textbf{Qwen2.5-0.5B-Instruct}~\citep{qwen2025qwen25}: Qwen2, 24 layers, $d_{\text{hidden}}=896$, 14 query heads, 2 KV heads, $d_{\text{head}}=64$. 4-bit quantised.
  \item \textbf{Qwen3-0.6B-Base}: Qwen3, 28 layers, $d_{\text{hidden}}=1024$, 16 query heads, 8 KV heads, $d_{\text{head}}=128$. Full precision.
  \item \textbf{DeepSeek-R1-Distill-Qwen-1.5B}~\citep{guo2025deepseek}: Qwen2 architecture (reasoning-distilled), 28 layers, $d_{\text{hidden}}=1536$, 12 query heads, 2 KV heads, $d_{\text{head}}=128$.
  \item \textbf{Qwen2.5-1.5B-Instruct}: Qwen2, 28 layers, $d_{\text{hidden}}=1536$, 12 query heads, 2 KV heads, $d_{\text{head}}=128$. 4-bit quantised.
  \item \textbf{Gemma 3-4B-IT}~\citep{team2024gemma}: Gemma3, 34 layers, $d_{\text{hidden}}=2560$, 8 query heads, 4 KV heads, $d_{\text{head}}=256$. 29/34 sliding window; 5 global. 4-bit quantised.
\end{itemize}
All models use pre-norm (RMSNorm) and RoPE. Experiments run on Apple M3~Max via MLX with bfloat16 precision.

\paragraph{Prompts.}
For Experiment~1 (KV reconstruction): ``The residual stream in a transformer is the central information highway. All attention and MLP outputs are additive updates to it.'' (24 tokens after tokenization.)

For Experiment~2 (generation match): ``Explain why the sky is blue in simple terms.'' Greedy (argmax) decoding with no temperature or sampling.

For Experiment~3 (multi-turn): System prompt ``You are a helpful, concise AI assistant.'' followed by user turns about France, the Eiffel Tower, etc. 30 tokens generated per turn.

\paragraph{Rank computation.}
The bilinear form $\bM^{(h)} = \bW_q^{(h)} {\bW_k^{(h)}}^\top$ was computed in float32 to avoid precision loss. Singular values were obtained via \texttt{numpy.linalg.svd}. Effective rank was defined as the smallest $r$ such that $\sum_{i=1}^r \sigma_i^2 \geq 0.90 \cdot \sum_{i=1}^{d_{\text{head}}} \sigma_i^2$.

\section{Additional Analysis}
\label{app:additional_analysis}

This section presents two supplementary analyses that complement the main results. Figure~\ref{fig:budget_sweep_app} examines how token match degrades under window-only caching as the KV budget varies, confirming that \kvdirect maintains lossless reconstruction at all budget levels. Figure~\ref{fig:layer_patching_app} visualises the cross-task residual patching experiment across all layers and models, providing layer-granularity evidence for the Markov property established in Section~\ref{sec:markov}.

\begin{figure*}
\centering
\includegraphics[width=\linewidth]{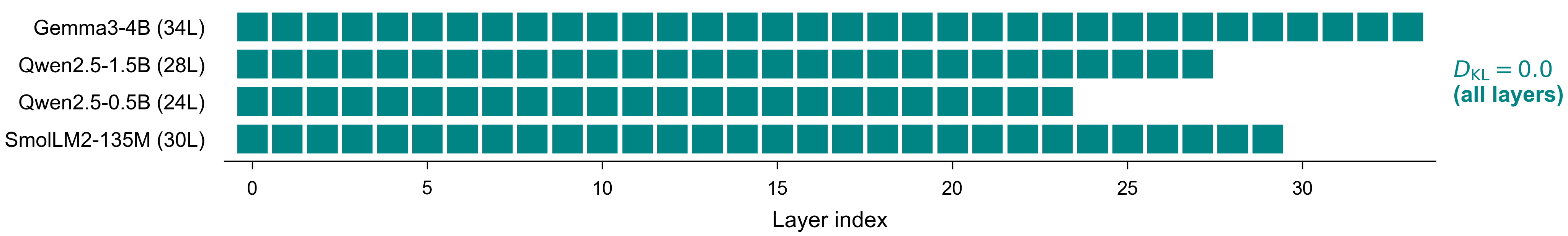}
\caption{Cross-task residual patching across all layers for four models. Each block represents one layer where the recipient's residual stream is replaced with the donor's. All tested layers produce $D_{\mathrm{KL}} = 0.0$ across all four architectures, confirming that the residual stream is a sufficient Markov state at every depth.}
\label{fig:layer_patching_app}
\end{figure*}

\begin{figure}
\centering
\includegraphics[width=\linewidth]{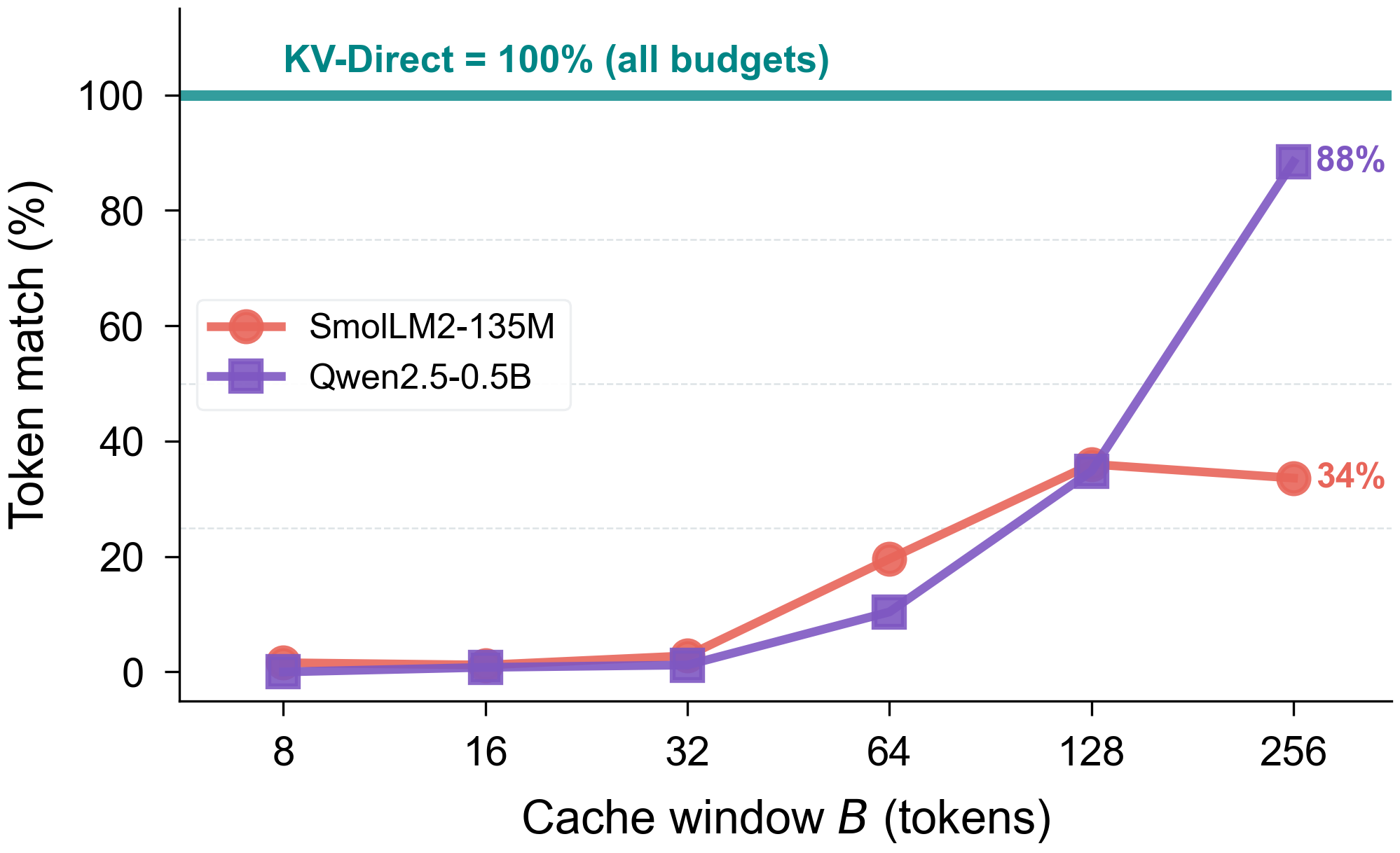}
\caption{KV budget sweep across six window sizes ($B \in \{8, 16, 32, 64, 128, 256\}$) with 250-token generation. (a) Window-only token match degrades sharply as the cache budget shrinks, while \kvdirect\ maintains 100\% match at all budgets. (b) Averaged across budgets, window-only caching achieves 16--23\% match versus \kvdirect's perfect recovery.}
\label{fig:budget_sweep_app}
\end{figure}

\end{document}